\documentclass[letterpaper, 10 pt, conference]{ieeeconf}  

\IEEEoverridecommandlockouts                              

\usepackage{moreverb,url}
\usepackage{graphics} 
\usepackage{epsfig} 
\usepackage{subcaption}

\usepackage{amssymb}  
\usepackage{mathtools}
\usepackage{commath}

\usepackage{amsthm}
\usepackage{algorithm}
\usepackage{algorithmic}

\newtheorem{thm}{Theorem}

\newtheorem{lem}[thm]{Lemma}

\newtheorem{remark}{Remark}

\usepackage{amsthm}
\usepackage{amsmath,amssymb,amsfonts}
\usepackage{graphicx}
\usepackage{subfloat}

\usepackage{url}

\usepackage{footnote}
\makesavenoteenv{table}

\allowdisplaybreaks

\title{\LARGE \bf
Robocentric Visual-Inertial Odometry
}

\author{Zheng Huai and Guoquan Huang
\thanks{The authors are with the Dept. of Mechanical Engineering,
        University of Delaware, Newark, DE 19716, USA
        {\tt\small \{zhuai|ghuang\}@udel.edu}.}%
}

\begin{document}

\maketitle

\begin{abstract}

In this paper, we propose a novel {\em robocentric} formulation of the visual-inertial navigation system (VINS)  within a sliding-window filtering framework and design an efficient, lightweight, {\em robocentric visual-inertial odometry} (R-VIO) algorithm for consistent motion tracking even in challenging environments using only a monocular camera and a 6-axis IMU. The key idea is to deliberately reformulate the VINS with respect to a moving local frame, rather than a fixed global frame of reference as in the standard world-centric VINS, in order to obtain relative motion estimates of higher accuracy for updating global poses. As an immediate advantage of this robocentric formulation, the proposed R-VIO can start from an arbitrary pose, {\em without} the need to align the initial orientation with the global gravitational direction. More importantly, we analytically show that the  linearized robocentric VINS does {\em not} undergo the observability mismatch issue as in the standard world-centric counterpart which was identified in the literature as the main cause of estimation inconsistency. Additionally, we investigate in-depth the special motions that degrade the performance in the world-centric formulation and show that such degenerate cases can be easily compensated in the proposed robocentric formulation, without resorting to additional sensors as in the world-centric formulation, thus leading to better robustness. The proposed R-VIO algorithm has been extensively tested through both Monte Carlo simulations and real-world experiments with different sensor platforms navigating in different environments, and shown to achieve better (or competitive at least) performance than the state-of-the-art VINS, in terms of consistency, accuracy and efficiency.

\end{abstract}

\section{Introduction}

Enabling high-precision, energy-efficient, and robust motion tracking in 3D on mobile devices and robots with minimal sensing 
holds potentially huge implications in many practical applications, ranging from mobile augmented reality to autonomous driving. 
To this end, inertial navigation offers a classical 3D localization solution which utilizes an inertial measurement unit (IMU) measuring the 3 degree-of-freedom (DOF) angular velocity and 3 DOF linear acceleration of the sensor platform on which it is rigidly attached. Typically, IMU works with a high frequency (e.g., 100Hz$\sim$1000Hz) that enables it to sense highly dynamic motion, while due to the corrupting sensor noise and bias, purely integrating IMU measurements may easily result in unusable motion estimates. 
This necessitates to utilize the aiding information from {\em at least} a single camera to reduce the accumulated inertial navigation drifts, which comes into the well-known
visual-inertial navigation system (VINS).

Over the past decade, significant progresses have been witnessed on the research and application of VINS, including the visual-inertial simultaneous localization and mapping (VI-SLAM) and the visual-inertial odometry (VIO), and many different VINS algorithms have been proposed (e.g., \cite{mourikis2007multi,jones2011visual,kelly2011visual,li2013high,leutenegger2015keyframe,shen2015tightly,usenko2016direct,mur2017visual} and references therein).
However, almost all these algorithms are based on the standard {\em world-centric} formulation -- that is, to estimate the absolute motion with respect to a fixed global frame of reference, such as the earth-centered earth-fixed (ECEF) or the north-east-down (NED) frame. In order to achieve accurate localization, such world-centric VINS algorithms usually require a particular initialization procedure to estimate the starting pose in the fixed global frame of reference, which, however, is hard to guarantee the accuracy in some cases (e.g., quick start, big sensor latency, or no/poor vision).
While the extended Kalman filter (EKF)-based world-centric VINS algorithms have the advantage of lower computational cost~\cite{mourikis2007multi,li2013high}
in comparing to the optimization-based iterative approaches (in which relinearization  incurs higher computation~\cite{leutenegger2015keyframe,shen2015tightly}), it may become {\em inconsistent},  primarily due to the fact that the EKF linearized systems have different observability properties from the corresponding underlying nonlinear systems~\cite{huang2010observability,martinelli2012vision,li2013high}. 
To address this issue, the remedies include enforcing the correct observabilty constraint~\cite{li2013high,hesch2014consistency,huang2014towards} 
or employing an invariant error representation~\cite{zhang2017convergence}.
%
%
However, one may ask: {\em Do we have to formulate VINS in the world-centric form?} The answer is {\em no}. 
Intuitively, considering how we navigate -- we might not remember the starting pose after traveling a long distance while knowing well the relative motion within a recent, short time interval; thus we may relax the fixed global frame of the VINS, instead, choosing a moving local frame as reference to better estimate relative motion which can be used for global pose update.

Notice that the usage of sensor-centered formulation for robot localization can be traced back to the 2D laser-based robocentric mapping~\cite{castellanos2004limits},
where the global frame is treated as a ``feature'' being observed from the moving robot frame 
and the odometry measurements are fused with the laser observations via EKF to estimate the relative motion, which is then used to update the global pose and shift the local frame of reference through a composition step when moving onto the next time step.
%
With a similar idea, \cite{civera20091p} used a camera-centered formulation to illustrate the potential of fusing visual information with the proprioceptive information, such as the angular and linear velocity measurements.
Both methods have been applied to the EKF-based SLAM while performing mapping with respect to a local frame, in this way the global uncertainty is properly limited thus improving the estimation consistency. 
It should also be noted that an EKF-based VINS algorithm with a different robocentric formulation and sensor-fusion scheme was recently introduced by~\cite{bloesch2015robust,bloesch2017iterated}.
Especially, its state vector includes the current IMU states, the observed features, as well as the sensor spatial calibration parameters, which are all expressed with respect to the current IMU frame; while the visual and inertial measurements are fused in a {\em direct} fashion. Moreover, in contrast to~\cite{castellanos2004limits,civera20091p}, this method directly estimates the absolute motion between the global frame and the local frame, and thus a standard {iterated} EKF is employed without the composition step used to shift the local frame of reference.

In this paper, 
we introduce a new  robocentric formulation of VINS  with respect to a local IMU frame of reference. Specifically, in contrast to~\cite{castellanos2004limits,civera20091p,bloesch2015robust,bloesch2017iterated} which keep the features in the state vector and would inevitably face the issue of ever-increasing computational cost as more features are observed and included, we focus on a sliding-window EKF-based robocentric VIO, akin to the multi-state constraint Kalman filter (MSCKF)~\cite{mourikis2007multi}.
In the proposed filter, the stochastic cloning~\cite{roumeliotis2002icra} is used for processing hundreds of features while only keeping a small number of relative robot poses (from which the features are observed) in the state vector, hence significantly reducing the computational cost. 
More importantly, the proposed robocentric system does not suffer from the observability mismatch issue as in the world-centric counterpart, thus having better consistency. 
In particular, the main contributions of the paper are summarized as follows:
\begin{itemize}

    \item We propose a novel robocentric VINS formulation  by reformulating the system with respect to a local IMU frame, 
    where both the global frame  treated as  the only ``feature'' 
    and the local gravity (i.e., with respect to the local frame of reference) are  included in the state vector. The local frame of reference is shifted at every image time through a {\em composition} step, 
   and the relative pose estimate between two consecutive local frames is used for updating the global pose estimate.
    
    \item We develop an efficient and robust R-VIO algorithm within a sliding-window filtering framework, 
    where a constant-size window of relative poses, instead of the observed features or the global poses, are included in the filter's state vector
    and are estimated by tightly fusing the camera and IMU measurements in a local frame of reference.
    %
    As such, a tailored {\em inverse depth}-based measurement model is developed to fully utilize such state configuration, where a dense connection is established between the feature measurements and the state considering the geometry between the feature and the poses from which it has been observed. 
    It should be pointed out that even if motionless, this model can still fuse the bearing information  from the distant features, which is particularly useful in reality.
    
    \item We study in-depth the observability properties of the proposed R-VIO, and analytically show that it has {\em constant} unobservable subspace, i.e., independent of the EKF linearization points, under generic motions. Thus, the resulting EKF-based robocentric VINS does not experience the observability mismatch that was identified as the main cause of estimation inconsistency \cite{huang2010observability,li2013high,hesch2014consistency}. More importantly, the proposed R-VIO system not only has correct unobservable dimensions, but also the desired unobservable directions. 
    Furthermore, we investigate the unobservable directions under degenerate motions, such as planar motion, and show that the possible performance degradation occurred in the world-centric formulation can be easily mitigated by the R-VIO {\em without} using the information of any additional sensor. 

    \item We perform extensive tests on both the Monte Carlo simulations and the real-world experiments that are running on different sensor platforms from the micro aerial vehicle (MAV) flying indoor to ground vehicle driving in dynamic traffic scenarios. All the real-time results  thoroughly validate the superior performance of the proposed R-VIO algorithm.
\end{itemize}


\section{Related work}
\label{sec:rw}

As mentioned earlier, the VINS algorithms generally include the VI-SLAM~\cite{lupton2012visual,leutenegger2015keyframe,shen2015tightly} and the VIO~\cite{mourikis2007multi,li2013high,forster2015manifold}. 
The former jointly estimates the feature positions and the camera/IMU pose that together form the state vector, whereas the latter does not include the features in the state but still utilizes the visual measurements to impose motion constraints between the camera/IMU poses. 
In general, by performing mapping, the VI-SLAM gains the better accuracy from the feature map and the possible loop closures while incurring higher computational complexity than the VIO, 
although different methods have been proposed to address this  issue (e.g., \cite{shen2015tightly,leutenegger2015keyframe,usenko2016direct,mur2017visual}). 
While there were also efforts to integrate VIO and SLAM~\cite{mourikis2009vision,li2013optimization},
in this paper we focus on the design of lightweight VIO that can serve as an essential building block for large-scale navigation systems.

There are different schemes available for VINS to fuse the visual and inertial measurements 
which can be broadly categorized into the {loosely-coupled} and the {tightly-coupled}. 
The former processes the visual and inertial measurements separately to infer their own motion constraints which are fused later (e.g., \cite{weiss2011real,kneip2011robust,indelman2013information}). 
Although this method is computationally efficient, the decoupling of visual and inertial constraints results in information loss. 
By contrast, the tightly-coupled approach directly fuses the visual and inertial measurements within a single process and  achieves  higher accuracy (e.g., \cite{mourikis2007multi,li2013high,forster2015manifold,shen2015tightly,leutenegger2015keyframe}). 
As the embedded computing and sensing technologies advance, the tightly-coupled VINS can now run in real time even on the resource-constrained sensor platforms such as MAVs and phones, thus becoming the methodological focus of this paper.

In particular, there are two main approaches for tightly-coupled state estimation, i.e., the optimization-based and the EKF-based. Typically, bundle adjustment (BA)~\cite{triggs1999bundle} is employed by the former that is to estimate all the states involved in all of the available measurements by solving a nonlinear least-squares problem (e.g., \cite{forster2015manifold,leutenegger2015keyframe}). As the relinearization of nonlinear measurement models is carried out at each iteration, this  would incur higher computational cost as compared to the EKF-based methods (e.g., \cite{mourikis2007multi,li2013high}). 
However, as what was mentioned before, the standard EKF-based VINS suffers from the estimation inconsistency primarily caused by the observability mismatch due to EKF linearization (e.g., \cite{huang2010observability,hesch2014consistency}). 
Recently, \cite{bloesch2015robust,bloesch2017iterated} introduced an EKF-based VINS solution using a robocentric formulation, 
which, however, follows the VI-SLAM framework and employs the iterated EKF update in a direct fashion.
In contrast to that, inspired by the robocentric mapping that improves the EKF consistency in the 2D SLAM~\cite{castellanos2004limits},
in this paper we propose a robocentric formulation within the sliding window filter-based VIO framework
and perform the observability analysis of the EKF-based robocentric VINS to theoretically support the consistency improvement of the proposed R-VIO algorithm.

\section{Estimator design}
\label{sec:rvio}

Consider a mobile platform equipped with an IMU and a single camera navigating in 3D environments. In contrast to the standard world-centric VINS using a fixed global frame of reference, $\{G\}$, in the proposed robocentric formulation, the frame $\{I\}$ affixed to IMU is set to be the immediate, local frame of reference for navigation, termed $\{R\}$. As a result, the global frame $\{G\}$ (or the first local frame of reference, $\{R_0\}$) turns into a ``moving" feature from the perspective of $\{R\}$; and during navigation, $\{R\}$ is transformed from one IMU frame to another. In this section, we deliberately reformulate the VINS problem with respect to such a moving local, rather than a fixed global, frame of reference, and present in detail the proposed R-VIO algorithm within a sliding-window filtering framework.

\subsection{State vector}

The state vector of the proposed robocentric VINS consists of two parts: (i) the global state that maintains the motion information of the starting frame $\{G\}$ (i.e., $\{R_0\}$), and (ii) the IMU state that characterizes the motion from the local frame of reference to the current IMU frame. In particular, at time-step $\tau\in[t_k,t_{k+1}]$ the state expressed in the local frame of reference, $\{R_k\}$, is given by:\footnote{Throughout this paper, $k,k+1,\ldots$ indicate the image time-steps, while $\tau,\tau+1,\ldots$ are the IMU time-steps between every two consecutive images. $\{I\}$ and $\{C\}$ denote the IMU frame and camera frame, respectively, $\{R\}$ is the robocentric frame of reference which is selected with the corresponding IMU frame at every image time-step. The subscript $\ell |i$ refers to the estimate of a quantity at time-step $\ell$, after all measurements up to time-step $i$ have been processed. $\hat{x}$ is used to denote the estimate of a random variable $x$, while $\tilde{x}=x-\hat{x}$ is the additive error in this estimate. $\mathbf{I}_n$ and $\mathbf{0}_n$ are the $n\times n$ identity and zero matrices, respectively. Finally, the left superscript denotes the frame of reference with respect to which the vector is expressed.}
\begin{equation}
    \begin{split}
    &{^{R_k}}\mathbf{x}_\tau =
    \begin{bmatrix}
    {^{R_k}}\mathbf{x}_G^\top & {^{R_k}}\mathbf{x}_{I_\tau}^\top
    \end{bmatrix}^\top, \\
    &{^{R_k}}\mathbf{x}_G =
    \begin{bmatrix}
    {^k_G}\bar{q}^\top & {^{R_k}}\mathbf{p}_G^\top & {^{R_k}}\mathbf{g}^\top
    \end{bmatrix}^\top, \\
    &{^{R_k}}\mathbf{x}_{I_\tau} =
    \begin{bmatrix}
    {^\tau_k}\bar{q}^\top & {^{R_k}}\mathbf{p}_{I_\tau}^\top & \mathbf{v}_{I_\tau}^\top & \mathbf{b}_{g_\tau}^\top & \mathbf{b}_{a_\tau}^\top
    \end{bmatrix}^\top
    \end{split}
    \label{eq:x}
\end{equation}
where ${^k_G}\bar{q}$ is the $4\times1$ unit quaternion~\cite{breckenridge1979jplq} describing the rotation from $\{G\}$ to $\{R_k\}$, ${^{R_k}}\mathbf{p}_G$ is the position of $\{G\}$ in $\{R_k\}$, ${^\tau_k}\bar{q}$ and ${^{R_k}}\mathbf{p}_{I_\tau}$ are the relative rotation and translation from $\{R_k\}$ to the current IMU frame, $\{I_\tau\}$, $\mathbf{v}_{I_\tau}$ is the local velocity expressed in $\{I_\tau\}$, and $\mathbf{b}_{g_\tau}$ and $\mathbf{b}_{a_\tau}$ denote the IMU's gyroscope and accelerometer biases, respectively. It is important to note that the local gravity, ${^{R_k}}{\mathbf{g}}$, is also included in the state vector. The corresponding error state is then given by:
\begin{equation}
    \begin{split}
    &{^{R_k}}\tilde{\mathbf{x}}_\tau =
    \begin{bmatrix}
    {^{R_k}}\tilde{\mathbf{x}}_G^\top & {^{R_k}}\tilde{\mathbf{x}}_{I_\tau}^\top
    \end{bmatrix}^\top, \\
    &{^{R_k}}\tilde{\mathbf{x}}_G =
    \begin{bmatrix}
    \delta{\boldsymbol{\theta}}_G^\top & {^{R_k}}\tilde{\mathbf{p}}_G^\top & {^{R_k}}\tilde{\mathbf{g}}^\top
    \end{bmatrix}^\top, \\
    &{^{R_k}}\tilde{\mathbf{x}}_{I_\tau} =
    \begin{bmatrix}
    \delta{\boldsymbol{\theta}}_\tau^\top & {^{R_k}}\tilde{\mathbf{p}}_{I_\tau}^\top & \tilde{\mathbf{v}}_{I_\tau}^\top & \tilde{\mathbf{b}}_{g_\tau}^\top & \tilde{\mathbf{b}}_{a_\tau}^\top
    \end{bmatrix}^\top
    \end{split}
    \label{eq:dx}
\end{equation}
In particular, the error quaternion is defined by $\bar{q}=\delta{\bar{q}}\otimes\hat{\bar{q}}$:
\begin{equation}
    \delta{\bar{q}} \simeq
    \begin{bmatrix}
    \frac{1}{2}\delta{\boldsymbol{\theta}}^\top & 1
    \end{bmatrix}^\top, \quad
    \mathbf{C}(\delta{\bar{q}}) = \mathbf{I}_3-\lfloor\delta{\boldsymbol{\theta}}\times\rfloor
    \label{eq:eq}
\end{equation}
where $\otimes$ denotes the quternion multiplication,  $\delta{\bar{q}}$ is the error quaternion associated with the 3DOF error angle $\delta{\boldsymbol{\theta}}$, $\mathbf{C}(\cdot)$ denotes a $3\times 3$ rotation matrix, and $\lfloor\cdot\times\rfloor$ is the skew-symmetric operator~\cite{trawny2005indirect}.

At time-step $k$ when the corresponding IMU frame, $\{I_k\}$, becomes the frame of reference (i.e., $\{R_k\}$) of estimation, a window of the relative poses between the last $N$ robocentric frames of reference is included in the state vector, as:
\begin{equation}
    \begin{split}
    &\hat{\mathbf{x}}_k =
    \begin{bmatrix}
    {^{R_k}}\hat{\mathbf{x}}_k^\top & \hat{\mathbf{w}}_k^\top
    \end{bmatrix}^\top,  \\
    &\hat{\mathbf{w}}_k =
    \begin{bmatrix}
    {^2_1}\hat{\bar{q}}^\top & {^{R_1}}\hat{\mathbf{p}}^\top_{R_2} & \ldots & {^N_{N-1}}\hat{\bar{q}}^\top & {^{R_{N-1}}}\hat{\mathbf{p}}^\top_{R_N}
    \end{bmatrix}^\top
    \end{split}
    \label{eq:xekf}
\end{equation}
where ${^i_{i-1}}\hat{\bar{q}}$ and ${^{R_{i-1}}}\hat{\mathbf{p}}_{R_i}$ express the relative rotation and translation from $\{R_{i-1}\}$ to $\{R_i\}$, $i=2,\ldots,N$. 
To keep the state vector of constant size over time, 
we manage it in the sliding-window fashion, i.e., marginalizing the oldest one when a new relative pose is included in the window. Accordingly, the augmented error state is given by:
\begin{equation}
    \begin{split}
    &\tilde{\mathbf{x}}_k =
    \begin{bmatrix}
    {^{R_k}}\tilde{\mathbf{x}}_k^\top & \tilde{\mathbf{w}}_k^\top
    \end{bmatrix}^\top,  \\
    &\tilde{\mathbf{w}}_k =
    \begin{bmatrix}
    \delta{\boldsymbol{\theta}}_2^\top & {^{R_1}}\tilde{\mathbf{p}}^\top_{R_2} & \ldots & \delta{\boldsymbol{\theta}}_N^\top & {^{R_{N-1}}}\tilde{\mathbf{p}}^\top_{R_N}
    \end{bmatrix}^\top
    \end{split}
    \label{eq:exekf}
\end{equation}

\subsection{Propagation}

We first present the motion model for the robocentric state, ${^{R_k}}{\mathbf{x}}_\tau$ (see \eqref{eq:x}), then extend it to the augmented state, ${\mathbf{x}}_\tau$ (see \eqref{eq:xekf}). Note that during the time interval $[t_k,t_{k+1}]$ the global frame is static with respect to the local frame of reference, $\{R_k\}$, i.e., ${^{R_k}}\dot{\tilde{\mathbf{x}}}_G = \mathbf{0}_{9\times1}$. For the IMU state, we introduce a locally-parameterized kinematic model: 
\begin{equation}
    \begin{split}
    &{^\tau_k}\dot{\bar{q}} = \frac{1}{2}\boldsymbol{\Omega}(\boldsymbol{\omega}){^\tau_k}\bar{q}, \;\; {^{R_k}}\dot{\mathbf{p}}_{I_\tau} = \mathbf{C}({^\tau_k}\bar{q})^\top\mathbf{v}_{I_\tau}, \\
    &\dot{\mathbf{v}}_{I_\tau} = {^\tau}\mathbf{a}-\lfloor\boldsymbol{\omega}\times\rfloor\mathbf{v}_{I_\tau}, \;\; \dot{\mathbf{b}}_g=\mathbf{n}_{wg}, \;\; \dot{\mathbf{b}}_a=\mathbf{n}_{wa}
    \end{split}
    \label{eq:xdot}
\end{equation}
where $\mathbf{n}_{wg}\sim\mathcal{N}(\mathbf{0},\sigma_{wg}^2\mathbf{I}_3)$ and $\mathbf{n}_{wa}\sim\mathcal{N}(\mathbf{0},\sigma_{wa}^2\mathbf{I}_3)$ are the zero-mean white Gaussian noise that drive the IMU biases,  and $\boldsymbol{\omega}$ and ${^\tau}\mathbf{a}$ are the angular velocity and linear acceleration expressed in $\{I_\tau\}$, respectively. 
And for $\boldsymbol{\omega}=[\omega_x,\omega_y,\omega_z]^\top$, we have:
\begin{equation}
    \boldsymbol{\Omega}\left(\boldsymbol{\omega}\right) =
    \begin{bmatrix}
    -\lfloor\boldsymbol{\omega}\times\rfloor & \boldsymbol{\omega} \\
    -\boldsymbol{\omega}^\top & 1
    \end{bmatrix}, \;
    \lfloor\boldsymbol{\omega}\times\rfloor =
    \begin{bmatrix}
    0 & -\omega_z & \omega_y \\
    \omega_z & 0 & -\omega_x \\
    -\omega_y & \omega_x & 0
    \end{bmatrix} \nonumber
\end{equation}

Typically, IMU provides the gyroscope and accelerometer measurements, $\boldsymbol{\omega}_m$ and $\mathbf{a}_m$, expressed in the IMU frame:
\begin{align}
    \boldsymbol{\omega}_m &= \boldsymbol{\omega}+\mathbf{b}_g+\mathbf{n}_g \label{eq:wm} \\
    \mathbf{a}_m &= {^I}\mathbf{a}+{^I}\mathbf{g}+\mathbf{b}_a+\mathbf{n}_a \label{eq:am}
\end{align}
where $\mathbf{n}_g\sim\mathcal{N}(\mathbf{0},\sigma_{g}^2\mathbf{I}_3)$ and $\mathbf{n}_a\sim\mathcal{N}(\mathbf{0},\sigma_{a}^2\mathbf{I}_3)$ are the zero-mean white Gaussian sensor noise, and ${^I}\mathbf{g}$ characterizes the gravity effect on the IMU frame.

Linearizing \eqref{eq:xdot} about the current state estimate yields the following continuous-time IMU state propagation:
\begin{equation}
    \begin{split}
    &{^\tau_k}\dot{\hat{\bar{q}}} = \frac{1}{2}\boldsymbol{\Omega}(\hat{\boldsymbol{\omega}}){^\tau_k}\hat{\bar{q}}, \;\;
    {^{R_k}}\dot{\hat{\mathbf{p}}}_{I_\tau} = {^\tau_k}\mathbf{C}_{\hat{\bar{q}}}^\top\hat{\mathbf{v}}_{I_\tau}, \\
    &\dot{\hat{\mathbf{v}}}_{I_\tau} = \hat{\mathbf{a}}-{^\tau}\hat{\mathbf{g}}-\lfloor\hat{\boldsymbol{\omega}}\times\rfloor\hat{\mathbf{v}}_{I_\tau}, \;\; \dot{\hat{\mathbf{b}}}_g = \mathbf{0}_{3\times1}, \;\; \dot{\hat{\mathbf{b}}}_a = \mathbf{0}_{3\times1}
    \end{split}
    \label{eq:xdothat}
\end{equation}
where for brevity we have denoted $\hat{\boldsymbol{\omega}}=\boldsymbol{\omega}_m-\hat{\mathbf{b}}_g$ and $\hat{\mathbf{a}}=\mathbf{a}_m-\hat{\mathbf{b}}_a$, ${^\tau_k}\mathbf{C}_{\hat{\bar{q}}}=\mathbf{C}({^\tau_k}\hat{\bar{q}})$, and ${^\tau}\hat{\mathbf{g}}={^\tau_k}\mathbf{C}_{\hat{\bar{q}}}{^{R_k}}\hat{\mathbf{g}}$. Accordingly, with both \eqref{eq:xdot} and \eqref{eq:xdothat}, we have continuous-time robocentric error-state model in the form of:
\begin{equation}
    {^{R_k}}\dot{\tilde{\mathbf{x}}}_\tau = \mathbf{F}{^{R_k}}\tilde{\mathbf{x}}_\tau+\mathbf{G}\mathbf{n}
\end{equation}
where $\mathbf{n}=[\mathbf{n}_g^\top \quad \mathbf{n}_{wg}^\top \quad \mathbf{n}_a^\top \quad \mathbf{n}_{wa}^\top]^\top$ is the IMU input noise vector, $\mathbf{F}$ is the robocentric error-state transition matrix, and $\mathbf{G}$ is the noise Jacobian, respectively (see \eqref{eq:sysms}).

\newcounter{tempEquationCounter} 
\newcounter{thisEquationNumber}
\newenvironment{floatEq}
{\setcounter{thisEquationNumber}{\value{equation}}\addtocounter{equation}{1}
\begin{figure*}[!t]
\normalsize\setcounter{tempEquationCounter}{\value{equation}}
\setcounter{equation}{\value{thisEquationNumber}}
}
{\setcounter{equation}{\value{tempEquationCounter}}
\hrulefill\vspace*{1pt}
\end{figure*}
}
\begin{floatEq}
\begin{equation}
    \mathbf{F} =
    \begin{bmatrix}
    \mathbf{0}_3& \mathbf{0}_3& \mathbf{0}_3& \mathbf{0}_3& \mathbf{0}_3& \mathbf{0}_3& \mathbf{0}_3& \mathbf{0}_3 \\
    \mathbf{0}_3& \mathbf{0}_3& \mathbf{0}_3& \mathbf{0}_3& \mathbf{0}_3& \mathbf{0}_3& \mathbf{0}_3& \mathbf{0}_3 \\
    \mathbf{0}_3& \mathbf{0}_3& \mathbf{0}_3& \mathbf{0}_3& \mathbf{0}_3& \mathbf{0}_3& \mathbf{0}_3& \mathbf{0}_3 \\
    \mathbf{0}_3& \mathbf{0}_3& \mathbf{0}_3&
    -\lfloor\hat{\boldsymbol{\omega}}\times\rfloor& \mathbf{0}_3& \mathbf{0}_3& -\mathbf{I}_3& \mathbf{0}_3 \\
    \mathbf{0}_3& \mathbf{0}_3& \mathbf{0}_3&
    -{^\tau_k}\mathbf{C}_{\hat{\bar{q}}}^\top\lfloor\hat{\mathbf{v}}_{I_\tau}\times\rfloor& \mathbf{0}_3& {^\tau_k}\mathbf{C}_{\hat{\bar{q}}}^\top& \mathbf{0}_3& \mathbf{0}_3 \\
    \mathbf{0}_3& \mathbf{0}_3& -{^\tau_k}\mathbf{C}_{\hat{\bar{q}}}&
    -\lfloor{^\tau}\hat{\mathbf{g}}\times\rfloor& \mathbf{0}_3& -\lfloor\hat{\boldsymbol{\omega}}\times\rfloor&
    -\lfloor\hat{\mathbf{v}}_{I_\tau}\times\rfloor& -\mathbf{I}_3 \\
    \mathbf{0}_3& \mathbf{0}_3& \mathbf{0}_3& \mathbf{0}_3& \mathbf{0}_3& \mathbf{0}_3& \mathbf{0}_3& \mathbf{0}_3 \\
    \mathbf{0}_3& \mathbf{0}_3& \mathbf{0}_3& \mathbf{0}_3& \mathbf{0}_3& \mathbf{0}_3& \mathbf{0}_3& \mathbf{0}_3
    \end{bmatrix}, \quad
    \mathbf{G} =
    \begin{bmatrix}
    \mathbf{0}_3& \mathbf{0}_3& \mathbf{0}_3& \mathbf{0}_3 \\
    \mathbf{0}_3& \mathbf{0}_3& \mathbf{0}_3& \mathbf{0}_3 \\
    \mathbf{0}_3& \mathbf{0}_3& \mathbf{0}_3& \mathbf{0}_3 \\
    -\mathbf{I}_3& \mathbf{0}_3& \mathbf{0}_3& \mathbf{0}_3 \\
    \mathbf{0}_3& \mathbf{0}_3& \mathbf{0}_3& \mathbf{0}_3 \\
    -\lfloor\hat{\mathbf{v}}_{I_\tau}\times\rfloor& \mathbf{0}_3& -\mathbf{I}_3& \mathbf{0}_3 \\
    \mathbf{0}_3& \mathbf{I}_3& \mathbf{0}_3& \mathbf{0}_3 \\
    \mathbf{0}_3& \mathbf{0}_3& \mathbf{0}_3& \mathbf{I}_3
    \end{bmatrix}
\label{eq:sysms}
\end{equation}
\end{floatEq}

For an actual implementation of EKF, the discrete-time propagation model is needed. First, the IMU state estimate, ${^{R_k}}\hat{\mathbf{x}}_{I_\tau}$, is obtained as follows: (i) by integrating \eqref{eq:xdothat} we have:
\begin{align}
    {^\tau_k}\hat{\bar{q}} &= \int^{t_\tau}_{t_k}{^s_k}\dot{\hat{\bar{q}}}\;{ds} \nonumber \\
    &= \int^{t_\tau}_{t_k}\frac{1}{2}\boldsymbol{\Omega}(\hat{\boldsymbol{\omega}}){^s_k}\hat{\bar{q}}\;{ds} \nonumber \\
    &= \int^{t_\tau}_{t_k}\frac{1}{2}\boldsymbol{\Omega}\big(\boldsymbol{\omega}_m-\hat{\mathbf{b}}_g\big){^s_k}\hat{\bar{q}}\;{ds}
\end{align}
which can be solved using zeroth order quaternion integrator \cite{trawny2005indirect}; (ii) ${^{R_k}}\hat{\mathbf{p}}_{I_\tau}$ and ${^{R_k}}\hat{\mathbf{v}}_{I_\tau}$ can be computed respectively using IMU preintegration, as:
\begin{align}
    {^{R_k}}\hat{\mathbf{p}}_{I_\tau} &= \hat{\mathbf{v}}_{I_k}\Delta{t}+\int^{t_\tau}_{t_k}\int^{s}_{t_k}{^\mu_k}\mathbf{C}_{\hat{\bar{q}}}^\top{^\mu}\hat{\mathbf{a}}\;{d\mu}{ds} \nonumber \\
    &= \hat{\mathbf{v}}_{I_k}\Delta{t}+\int^{t_\tau}_{t_k}\int^{s}_{t_k}{^\mu_k}\mathbf{C}_{\hat{\bar{q}}}^\top\big({^\mu}\mathbf{a}_m-\hat{\mathbf{b}}_a-{^\mu}\hat{\mathbf{g}}\big)\;{d\mu}{ds} \nonumber \\
    &= \hat{\mathbf{v}}_{I_k}\Delta{t}-\frac{1}{2}{^{R_k}}\hat{\mathbf{g}}\Delta{t}^2 \nonumber \\
    &\quad+\underbrace{\int^{t_\tau}_{t_k}\int^{s}_{t_k}{^\mu_k}\mathbf{C}_{\hat{\bar{q}}}^\top\big({^\mu}\mathbf{a}_m-\hat{\mathbf{b}}_a\big)\;{d\mu}{ds}}_{\Delta{\mathbf{p}}_{k,\tau}} \\
    {^{R_k}}\hat{\mathbf{v}}_{I_\tau} &= \hat{\mathbf{v}}_{I_k}+\int^{t_\tau}_{t_k}{^s_k}\mathbf{C}_{\hat{\bar{q}}}^\top{^s}\hat{\mathbf{a}}\;{ds} \nonumber \\
    &= \hat{\mathbf{v}}_{I_k}+\int^{t_\tau}_{t_k}{^s_k}\mathbf{C}_{\hat{\bar{q}}}^\top\big({^s}\mathbf{a}_m-\hat{\mathbf{b}}_a-{^s}\hat{\mathbf{g}}\big)\;{ds} \nonumber \\
    &= \hat{\mathbf{v}}_{I_k}-{^{R_k}}\hat{\mathbf{g}}\Delta{t}+
    \underbrace{\int^{t_\tau}_{t_k}{^s_k}\mathbf{C}_{\hat{\bar{q}}}^\top\big({^s}\mathbf{a}_m-\hat{\mathbf{b}}_a\big)\;{ds}}_{\Delta{\mathbf{v}}_{k,\tau}}
\end{align}
where $\Delta{t}=t_\tau-t_k$. Especially, the preintegrated terms, $\Delta{\mathbf{p}}$ and $\Delta{\mathbf{v}}$, can be recursively computed with all the incoming IMU measurements~\cite{Eckenhoff2016WAFR}.
Therefore, the estimate of velocity in the current IMU frame, $\hat{\mathbf{v}}_{I_\tau}$, can be obtained as $\hat{\mathbf{v}}_{I_\tau} = {^\tau_k}\mathbf{C}_{\hat{\bar{q}}}{^{R_k}}\hat{\mathbf{v}}_{I_\tau}$; (iii) assume the bias estimates are constant over the time interval $[t_k,t_{k+1}]$: $\hat{\mathbf{b}}_g=\hat{\mathbf{b}}_{g_k}$ and $\hat{\mathbf{b}}_a=\hat{\mathbf{b}}_{a_k}$ for both (i) and (ii).

Then, for covariance propagation, the discrete-time error-state transition matrix $\boldsymbol{\Phi}(t_{\tau+1},t_\tau)$ can be obtained using the forward Euler method over the time interval $[t_\tau,t_{\tau+1}]$:
\begin{equation}
    \boldsymbol{\Phi}(t_{\tau+1},t_\tau) = \exp(\mathbf{F}\delta{t}) \simeq \mathbf{I}_{24}+\mathbf{F}\delta{t} =: \boldsymbol{\Phi}_{\tau+1,\tau}
    \label{eq:phid}
\end{equation}
where $\delta{t}=t_{\tau+1}-t_\tau$. It results in the covariance propagation starting from $\mathbf{P}_k$ (not $\mathbf{P}_{k|k}$) at time-step $k$:
\begin{equation}
    \mathbf{P}_{\tau+1|k} = \boldsymbol{\Phi}_{\tau+1,\tau}\mathbf{P}_{\tau|k}\boldsymbol{\Phi}_{\tau+1,\tau}^\top+\mathbf{G}\boldsymbol{\Sigma}\mathbf{G}^\top\delta{t}
    \label{eq:covprop}
\end{equation}
where $\boldsymbol{\Sigma}={\mathbf{Diag}}\left[\sigma_g^2\mathbf{I}_3 \quad \sigma_{wg}^2\mathbf{I}_3 \quad \sigma_a^2\mathbf{I}_3 \quad \sigma_{wa}^2\mathbf{I}_3 \right]$ denotes the continuous-time input noise covariance matrix, and the detailed derivations can be found in our companion technical report~\cite{supp}.

For the augmented state, $\hat{\mathbf{x}}_k$, we consider that the relative poses in the sliding window are static, i.e., $\hat{\mathbf{w}}_\tau=\hat{\mathbf{w}}_k$, and the corresponding augmented covariance matrix, $\mathbf{P}_k$, can be partitioned according to the robocentric state and the sliding-window state (see \eqref{eq:xekf}), as:
\begin{equation}
\mathbf{P}_k = \begin{bmatrix}
\mathbf{P}_{{\mathbf{x}\mathbf{x}}_k} & \mathbf{P}_{{\mathbf{x}\mathbf{w}}_k} \\
\mathbf{P}_{{\mathbf{x}\mathbf{w}}_k}^\top & \mathbf{P}_{{\mathbf{w}\mathbf{w}}_k}
\end{bmatrix}
\end{equation}
The propagated covariance at time-step $\tau+1$ is given by:
\begin{equation}
    \mathbf{P}_{\tau+1|k} =
    \begin{bmatrix}
    \mathbf{P}_{{\mathbf{x}\mathbf{x}}_{\tau+1|k}} & \boldsymbol{\Phi}_{\tau+1,k}\mathbf{P}_{{\mathbf{x}\mathbf{w}}_k} \\
    \mathbf{P}_{{\mathbf{x}\mathbf{w}}_k}^\top\boldsymbol{\Phi}_{\tau+1,k}^\top & \mathbf{P}_{{\mathbf{w}\mathbf{w}}_k}
    \end{bmatrix}
\end{equation}
where $\mathbf{P}_{{\mathbf{x}\mathbf{x}}_{\tau+1|k}}$ can be recursively computed using \eqref{eq:covprop}, and the compound error-state transition matrix is computed as: 
\begin{equation}
\boldsymbol{\Phi}_{\tau+1,k} = \prod\limits_{\ell=k}^{\tau}\boldsymbol{\Phi}_{\ell+\delta{t},\ell}
\end{equation}
with initial condition $\boldsymbol{\Phi}_{k,k}=\mathbf{I}_{24}$.

\subsection{Update}

\subsubsection{\bf Inverse-depth measurement model}

We adopt the {\em inverse depth} parameterization \cite{civera2008inverse} for the landmarks observed by a monocular camera, while being tailored for the proposed R-VIO. Assuming a single landmark, $L_j$, that has been observed from a set of $n_j$ robocentric frames, $\mathcal{R}_j$, the measurement of $L_j$ in the set of $n_j$ corresponding camera frames, $\mathcal{C}_j$, is given by the following perspective projection model with the $xyz$ coordinates ($i\in\mathcal{C}_j$):
\begin{equation}
    \mathbf{z}_{j,i} = \frac{1}{z^i_j}
    \begin{bmatrix}
    x^i_j \\ y^i_j
    \end{bmatrix}+\mathbf{n}_{j,i}, \quad
    {^{C_i}}\mathbf{p}_{L_j} =
    \begin{bmatrix}
    x^i_j & y^i_j & z^i_j
    \end{bmatrix}^\top
    \label{eq:zm}
\end{equation}
where $\mathbf{n}_{j,i}\sim\mathcal{N}(\mathbf{0},\sigma_{im}^2\mathbf{I}_2)$ is an additive image noise, and ${^{C_i}}\mathbf{p}_{L_j}$ denotes the position of $L_j$ in the camera frame $\{C_i\}$. The inverse-depth form for ${^{C_i}}\mathbf{p}_{L_j}$ can be written as:
\begin{equation}
    {^{C_i}}\mathbf{p}_{L_j}  =
    {^i_1}\bar{\mathbf{C}}_{\bar{q}}{^{C_1}}\mathbf{p}_{L_j}+{^i}\bar{\mathbf{p}}_1 =: \mathbf{f}_i(\phi,\psi,\rho) \nonumber
\end{equation}
\begin{equation}
    {^{C_1}}\mathbf{p}_{L_j} =
    \frac{1}{\rho}\mathbf{e}(\phi,\psi), \quad
    \mathbf{e} =
    \begin{bmatrix}
    \cos\phi\sin\psi \\ \sin\phi \\ \cos\phi\cos\psi
    \end{bmatrix}
    \label{eq:pinv}
\end{equation}
where ${^{C_1}}\mathbf{p}_{L_j}$ is the position of $L_j$ in the first camera frame of $\mathcal{C}_j$, $\mathbf{e}$ is the directional vector with $\phi$ and $\psi$ the elevation and azimuth expressed in $\{C_1\}$, and $\rho$ is the inverse depth along $\mathbf{e}$. In particular, the relative poses between $\{C_1\}$ and $\{C_i\}$, $i=2,\ldots,n_j$, are expressed using the camera-to-IMU calibration parameters, $\{{^C_I}\bar{q},{^C}\mathbf{p}_I\}$, and the sliding-window state, $\mathbf{w}$, as:
\begin{align}
    {^i_1}\bar{\mathbf{C}}_{\bar{q}} &=
    {^C_I}\mathbf{C}_{\bar{q}}{^i_1}\mathbf{C}_{\bar{q}}{^I_C}\mathbf{C}_{\bar{q}} \\
    {^i}\bar{\mathbf{p}}_1 &=
    {^C_I}\mathbf{C}_{\bar{q}}{^i_1}\mathbf{C}_{\bar{q}}{^I}\mathbf{p}_C+{^C_I}\mathbf{C}_{\bar{q}}{^{R_i}}\mathbf{p}_{R_1}+{^C}\mathbf{p}_I
    \label{eq:pc}
\end{align}
where we have used the following identities ($n=2,\ldots,i$):
\begin{align}
    {^i_1}\mathbf{C}_{\bar{q}} &= {^i_{i-1}}\mathbf{C}_{\bar{q}}{^{i-1}_{i-2}}\mathbf{C}_{\bar{q}}\ldots{^n_{n-1}}\mathbf{C}_{\bar{q}}\ldots{^2_1}\mathbf{C}_{\bar{q}} \\
    {^{R_i}}\mathbf{p}_{R_1} &= -\big({^i_{i-1}}\mathbf{C}_{\bar{q}}{^{R_{i-1}}}\mathbf{p}_{R_i}+{^i_{i-2}}\mathbf{C}_{\bar{q}}{^{R_{i-2}}}\mathbf{p}_{R_{i-1}}+\ldots \nonumber \\
    &\qquad+{^i_{n-1}}\mathbf{C}_{\bar{q}}{^{R_{n-1}}}\mathbf{p}_{R_n}+\ldots+{^i_1}\mathbf{C}_{\bar{q}}{^{R_1}}\mathbf{p}_{R_2}\big)
\end{align}
Interestingly, if the landmark is at infinity (i.e., $\rho\rightarrow0$), we can normalize \eqref{eq:pinv} by premultiplying $\rho$ to avoid potential numerical issues, as:
\begin{align}
    \rho{{^{C_i}}\mathbf{p}_{L_j}}  
    &= {^i_1}\bar{\mathbf{C}}_{\bar{q}}\mathbf{e}(\phi,\psi)+\rho{^i}\bar{\mathbf{p}}_1 \nonumber \\
    &=: \mathbf{h}_i(\mathbf{w},\phi,\psi,\rho) 
    =
    \begin{bmatrix}
    h_{i,1}(\mathbf{w},\phi,\psi,\rho) \\ h_{i,2}(\mathbf{w},\phi,\psi,\rho) \\ h_{i,3}(\mathbf{w},\phi,\psi,\rho)
    \end{bmatrix} 
    \label{eq:hinv}
\end{align}
Note that, this equation reserves the perspective geometry of \eqref{eq:pinv} while encompassing two degenerate cases: (i) observing the landmarks at infinity (i.e., $\rho\rightarrow0$), and (ii) having low parallax between two camera poses (i.e., ${^i}\bar{\mathbf{p}}_1\rightarrow0$). For both cases, \eqref{eq:hinv} can be approximated by $\mathbf{h}_i\simeq{^i_1}\bar{\mathbf{C}}_{\bar{q}}\mathbf{e}(\phi,\psi)$, and hence the corresponding measurements can still provide the information about the camera orientation.

Therefore, we introduce the following inverse depth-based measurement model for the proposed R-VIO:
\begin{equation}
    \mathbf{z}_{j,i} = \frac{1}{h_{i,3}(\mathbf{w},\phi,\psi,\rho)}
    \begin{bmatrix}
    h_{i,1}(\mathbf{w},\phi,\psi,\rho) \\ h_{i,2}(\mathbf{w},\phi,\psi,\rho)
    \end{bmatrix}+\mathbf{n}_{j,i}
    \label{eq:zinv}
\end{equation}
Denoting $\boldsymbol{\lambda}=[\phi,\psi,\rho]^\top$ and linearizing \eqref{eq:zinv} at the current state estimates, $\hat{\mathbf{x}}$ and $\hat{\boldsymbol{\lambda}}$, we have the following measurement residual equation:
\begin{equation}
    \mathbf{r}_{j,i} = \mathbf{z}_{j,i}-\hat{\mathbf{z}}_{j,i} \simeq \mathbf{H}_{\mathbf{x}_{j,i}}\tilde{\mathbf{x}}+\mathbf{H}_{\boldsymbol{\lambda}_{j,i}}\tilde{\boldsymbol{\lambda}}+\mathbf{n}_{j,i} \nonumber
\end{equation}
where
\begin{equation}
    \begin{split}
    &\mathbf{H}_{\mathbf{x}_{j,i}} = \mathbf{H}_{\text{p}_{j,i}}
    \begin{bmatrix}
    \mathbf{0}_{3\times24} &\!\!
    \ldots &\!\!
    \mathbf{H}_{\mathbf{w}_{j,i}} &\!\!
    \ldots &\!\!
    \end{bmatrix}, \\
    &\mathbf{H}_{\boldsymbol{\lambda}_{j,i}} = \mathbf{H}_{\text{p}_{j,i}}\mathbf{H}_{\text{inv}_{j,i}}, \\
    &\mathbf{H}_{\text{p}_{j,i}} = \frac{1}{\hat{h}_{i,3}}
    \begin{bmatrix}
    1 & 0 & -\frac{\hat{h}_{i,1}}{\hat{h}_{i,3}} \\
    0 & 1 & -\frac{\hat{h}_{i,2}}{\hat{h}_{i,3}}
    \end{bmatrix}, \\
    &\mathbf{H}_{\text{inv}_{j,i}} =
    \frac{\partial{\mathbf{h}_i}}{\partial{\tilde{\boldsymbol{\lambda}}}} =
    \begin{bmatrix}
    \frac{\partial{\mathbf{h}_i}}{\partial{[\tilde{\phi},\tilde{\psi}]^\top}} & \frac{\partial{\mathbf{h}_i}}{\partial{\tilde{\rho}}}
    \end{bmatrix} \\
    &\qquad\;\;= \begin{bmatrix}
    {^i_1}\bar{\mathbf{C}}_{\hat{\bar{q}}}
    \begin{bmatrix}
    -\sin\hat{\phi}\sin\hat{\psi} & \cos\hat{\phi}\cos\hat{\psi} \\
    \cos\hat{\phi} & 0 \\
    -\sin\hat{\phi}\cos\hat{\psi} & -\cos\hat{\phi}\sin\hat{\psi}
    \end{bmatrix} &
    {^i}\hat{\bar{\mathbf{p}}}_1
    \end{bmatrix}, \\
    &\mathbf{H}_{\mathbf{w}_{j,i}} =
    \frac{\partial{\mathbf{h}_i}}{\partial{\tilde{\mathbf{w}}}} =
    \begin{bmatrix}
    \frac{\partial{\mathbf{h}_i}}{\partial{\delta{\boldsymbol{\theta}}_2}} & 
    \frac{\partial{\mathbf{h}_i}}{\partial{{^{R_1}}\tilde{\mathbf{p}}}_{R_2}} &\!\! \ldots &\!\!
    \frac{\partial{\mathbf{h}_i}}{\partial{\delta{\boldsymbol{\theta}}_i}} & 
    \frac{\partial{\mathbf{h}_i}}{\partial{{^{R_{i-1}}}\tilde{\mathbf{p}}}_{R_i}}
    \end{bmatrix} \\
    &\frac{\partial{\mathbf{h}_i}}{\partial{\delta{\boldsymbol{\theta}}_n}} =
    {^C_I}\mathbf{C}_{\bar{q}}{^i_1}\mathbf{C}_{\hat{\bar{q}}}\lfloor\big({^I_C}\mathbf{C}_{\bar{q}}\hat{\mathbf{e}}+\hat{\rho}{^I}\mathbf{p}_C-\hat{\rho}{^{R_1}}\hat{\mathbf{p}}_{R_n}\big)\times\rfloor{^n_1}\mathbf{C}_{\hat{\bar{q}}}^\top, \\
    &\frac{\partial{\mathbf{h}_i}}{\partial{{^{R_{n-1}}}\tilde{\mathbf{p}}_{R_n}}} =
    -\hat{\rho}{^C_I}\mathbf{C}_{\bar{q}}{^i_{n-1}}\mathbf{C}_{\hat{\bar{q}}}, \quad n=2,\ldots,i.
    \end{split}
    \label{eq:Jacob}
\end{equation}
Specifically, $\mathbf{H}_{\mathbf{x}_{j,i}}$ and $\mathbf{H}_{\boldsymbol{\lambda}_{j,i}}$ are the Jacobians with respect to the vectors of state and inverse depth, respectively. Note that, through the Jacobian $\mathbf{H}_{\mathbf{w}_{j,i}}$ each measurements of $L_j$ is correlated to a sequence of relative poses in $\mathbf{w}$, building up a {\em dense} connection between the measurements and the state, however, without increasing the computational complexity. This is also different from~\cite{mourikis2007multi} where each measurement is only correlated to the global pose from which it is observed. Since an estimate of $\boldsymbol{\lambda}$ is needed for computing $\hat{\mathbf{z}}_{j,i}$ and $\mathbf{H}_{\boldsymbol{\lambda}_{j,i}}$, a local BA is firstly solved using the measurements, $\mathbf{z}_{j,i}$, $i\in\mathcal{C}_j$, and the relative pose estimates, $\hat{\mathbf{w}}$ (see Appendix A). After stacking the residuals $\mathbf{r}_{j,i}$, $i\in\mathcal{C}_j$, we obtain:
\begin{equation}
    \mathbf{r}_j \simeq \mathbf{H}_{\mathbf{x}_j}\tilde{\mathbf{x}}+\mathbf{H}_{\boldsymbol{\lambda}_j}\tilde{\boldsymbol{\lambda}}+\mathbf{n}_j
    \label{eq:rinv}
\end{equation}
Assuming the measurements obtained from different camera poses are independent, the covariance matrix of $\mathbf{n}_j$ is hence $\mathbf{R}_j=\sigma_{im}^2\mathbf{I}_{2n_j}$. As $\hat{\mathbf{x}}$ (precisely, $\hat{\mathbf{w}}$) is used to compute $\hat{\boldsymbol{\lambda}}$, the inverse-depth error, $\tilde{\boldsymbol{\lambda}}$, is correlated to $\tilde{\mathbf{x}}$. In order to find a valid residual for EKF update, we project  \eqref{eq:rinv} to the left nullspace of $\mathbf{H}_{\boldsymbol{\lambda}_j}$ (i.e., $\mathbf{O}_{\boldsymbol{\lambda}_j}^\top\mathbf{H}_{\boldsymbol{\lambda}_j}=\mathbf{0}$, and $\mathbf{O}_{\boldsymbol{\lambda}_j}^\top\mathbf{O}_{\boldsymbol{\lambda}_j}=\mathbf{I}$):
\begin{align}
    \bar{\mathbf{r}}_{j} = \mathbf{O}_{\boldsymbol{\lambda}_j}^\top\mathbf{r}_j \simeq \mathbf{O}_{\boldsymbol{\lambda}_j}^\top\mathbf{H}_{\mathbf{x}_j}\tilde{\mathbf{x}}+\mathbf{O}_{\boldsymbol{\lambda}_j}^\top\mathbf{n}_j = \bar{\mathbf{H}}_{\mathbf{x}_j}\tilde{\mathbf{x}}+\bar{\mathbf{n}}_j
    \label{eq:ol}
\end{align}
In general, $\mathbf{H}_{\boldsymbol{\lambda}_j}$ is $2n_j\times3$ matrix with full column rank and the nullspace of dimension $2n_j-3$, which can be efficiently computed, for example, using the {\em Givens rotations} \cite{golub2012matrix}, with $O(n_j^2)$ complexity. 
Since $\mathbf{O}_{\boldsymbol{\lambda}_j}$ is unitary, the covariance matrix of $\bar{\mathbf{n}}_j$ becomes:
\begin{equation}
    \bar{\mathbf{R}}_j=\mathbf{O}_{\boldsymbol{\lambda}_j}^\top\mathbf{R}_j\mathbf{O}_{\boldsymbol{\lambda}_j}=\sigma_{im}^2\mathbf{I}_{2n_j-3}
    \label{eq:Rj}
\end{equation}

At this point, let us examine some special cases where $\mathbf{H}_{\boldsymbol{\lambda}_{j,i}}$ (equivalently, $\mathbf{H}_{p_{j,i}}$ or $\mathbf{H}_{\text{inv}_{j,i}}$) becomes rank deficient (see \eqref{eq:Jacob}), which would affect computing the residual~\eqref{eq:rinv}.
First of all, if $\mathbf{H}_{p_{j,i}}$ becomes rank deficient, then we find two possible causes about $\hat{\mathbf{h}}_i$: (i) $\hat{h}_{i,1}=\hat{h}_{i,2}=\hat{h}_{i,3}$, which means that the image size should be at least $2f\times2f$ ($f$ is the focal length), or (ii) $\hat{h}_{i,1}\rightarrow0$ and $\hat{h}_{i,2}\rightarrow0$, which means that the measurement of $L_j$ is close to the principal point of camera image.
Secondly, if $\mathbf{H}_{\text{inv}_{j,i}}$ is rank deficient, we can also find two possible causes: (iii) $\cos\hat{\phi}\rightarrow0$, which means that we have either infinitely small focal length or infinitely large image size for the camera so that $\lvert\hat{\phi}\rvert\rightarrow\pi/2$ can happen, or (iv) ${^i}\hat{\bar{\mathbf{p}}}_1\rightarrow0$, which means a small parallax between $\{C_1\}$ and $\{C_i\}$. 
Among these causes, (i) is about the selection of the lens which must be restricted by the camera image size, and (iii) is too ideal to be realized in the real world; while (ii) and (iv) are common in the visual navigation which can be effectively detected by checking the values of pixel measurements and relative pose estimates, respectively. 
Therefore, we can discard the measurements that meet (ii) when computing the Jacobians. 
However, in the case (iv) (e.g., pure rotation or motionless), since the last column of $\mathbf{H}_{\boldsymbol{\lambda}_{j,i}}$ (and hence $\mathbf{H}_{\boldsymbol{\lambda}_j}$) approaches zero, we perform the Givens rotations only for the first two columns of $\mathbf{H}_{\boldsymbol{\lambda}_j}$ to guarantee a valid nullspace projection numerically (see \eqref{eq:ol}), and thus the dimension of $\bar{\mathbf{r}}_j$ increases by one (see \eqref{eq:Rj}). 
In addition, before EKF update, the {\em Mahalanobis distance} for each landmark is checked using all the measurements,
serving as the probabilistic outlier rejection:
\begin{equation}
    \mathcal{D}_j = \bar{\mathbf{r}}_j^\top\big(\bar{\mathbf{H}}_{\mathbf{x}_j}\mathbf{P}\bar{\mathbf{H}}_{\mathbf{x}_j}^\top+\bar{\mathbf{R}}_j\big)^{-1}\bar{\mathbf{r}}_j \leq \chi_{r,1-\alpha}^2
    \label{eq:maha}
\end{equation}
where $\chi_{r,1-\alpha}^2$ is a threshold obtained from the $\chi^2$ distribution with $r=\text{dim}(\bar{\mathbf{r}}_j)$, and $\alpha$ the significance level (e.g., 0.05). If \eqref{eq:maha} holds, then landmark $L_j$ is accepted as an inlier and used for EKF update.

\subsubsection{\bf EKF update}

Assuming that at time-step $k+1$ we have the measurements of $M$ landmarks to process, 
we can stack the resulting $\bar{\mathbf{r}}_j$, $j=1,\ldots,M$, to have:
\begin{equation}
    \bar{\mathbf{r}} = \bar{\mathbf{H}}_{\mathbf{x}}\tilde{\mathbf{x}}+\bar{\mathbf{n}}
    \label{eq:r}
\end{equation}
which is of dimension $d=\sum_{j=1}^M(2n_j-3)$. However, in practice, $d$ could be a large number even if $M$ is small (e.g., $d=170$, if 10 landmarks are observed from 10 robot poses). To reduce the computational complexity, QR decomposition is applied to \eqref{eq:r} to compress the dimension of measurement model. Note that, $\bar{\mathbf{H}}_{\mathbf{x}}$ is {\em rank deficient} with the zero columns corresponding to the robocentric state, while the nonzero columns corresponding to the states of relative poses in the sliding window are linearly independent. Therefore, to save the computational cost the QR decomposition can be applied to the nonzero part of $\bar{\mathbf{H}}_{\mathbf{x}}$ only, as:
\begin{align}
    \bar{\mathbf{H}}_{\mathbf{x}} &=
    \begin{bmatrix}
    \mathbf{0}_{d\times24} & \bar{\mathbf{H}}_\mathbf{w}
    \end{bmatrix} \nonumber \\
    &=
    \begin{bmatrix}
    \mathbf{0}_{d\times24} &
    \begin{bmatrix}
    \mathbf{Q}_1 & \mathbf{Q}_2
    \end{bmatrix}
    \begin{bmatrix}
    \bar{\mathbf{T}}_\mathbf{w} \\ \mathbf{0}_{(d-6(N-1))\times6(N-1)}
    \end{bmatrix}
    \end{bmatrix} \nonumber \\
    &=
    \begin{bmatrix}
    \mathbf{Q}_1 & \mathbf{Q}_2
    \end{bmatrix}
    \begin{bmatrix}
    \mathbf{0}_{d\times24} &
    \begin{bmatrix}
    \bar{\mathbf{T}}_\mathbf{w} \\
    \mathbf{0}_{(d-6(N-1))\times6(N-1)}
    \end{bmatrix}
    \end{bmatrix} \nonumber
\end{align}
where $\mathbf{Q}_1$ and $\mathbf{Q}_2$ are the unitary matrices of dimension $d\times6(N-1)$ and $d\times(d-6(N-1))$, respectively, and $\bar{\mathbf{T}}_{\mathbf{w}}$ is an upper triangular matrix of dimension $6(N-1)$. With this definition, \eqref{eq:r} yields:
\begin{align}
    &\;\bar{\mathbf{r}} =
    \begin{bmatrix}
    \mathbf{Q}_1 & \mathbf{Q}_2
    \end{bmatrix}
    \begin{bmatrix}
    \mathbf{0} & \bar{\mathbf{T}}_\mathbf{w} \\
    \mathbf{0} & \mathbf{0}
    \end{bmatrix}
    \tilde{\mathbf{x}}
    +\bar{\mathbf{n}} \Rightarrow \nonumber \\
    &\begin{bmatrix}
    \mathbf{Q}_1^\top \\ \mathbf{Q}_2^\top
    \end{bmatrix}\bar{\mathbf{r}} =
    \begin{bmatrix}
    \mathbf{0} & \bar{\mathbf{T}}_\mathbf{w} \\
    \mathbf{0} & \mathbf{0}
    \end{bmatrix}\tilde{\mathbf{x}}
    +\begin{bmatrix}
    \mathbf{Q}_1^\top \\ \mathbf{Q}_2^\top
    \end{bmatrix}\bar{\mathbf{n}}
\end{align}
for which, we discard the lower $d-6(N-1)$ rows which are only about the measurement noise, but employ the upper $6(N-1)$ rows, instead of \eqref{eq:r}, as the residual for the EKF update:
\begin{equation}
    \breve{\mathbf{r}} = \mathbf{Q}_1^\top\bar{\mathbf{r}} =
    \begin{bmatrix}
    \mathbf{0} & \bar{\mathbf{T}}_\mathbf{w}
    \end{bmatrix}\tilde{\mathbf{x}}
    +\mathbf{Q}_1^\top\bar{\mathbf{n}} = \breve{\mathbf{H}}_\mathbf{x}\tilde{\mathbf{x}}+\breve{\mathbf{n}}
\end{equation}
where $\breve{\mathbf{n}}=\mathbf{Q}_1^\top\bar{\mathbf{n}}$ is the noise vector with covariance matrix $\breve{\mathbf{R}}=\mathbf{Q}_1^\top\bar{\mathbf{R}}\mathbf{Q}_1=\sigma_{im}^2\mathbf{I}_{6(N-1)}$. In particular, when we have $d\gg6(N-1)$ these can be done using the Givens rotations, with $O(N^2d)$ complexity. 
Based on that, the standard EKF update is performed as follows~\cite{Maybeck1979}:
\begin{align*}
    &\mathbf{K} = \mathbf{P}\breve{\mathbf{H}}_\mathbf{x}^\top\big(\breve{\mathbf{H}}_\mathbf{x}\mathbf{P}\breve{\mathbf{H}}_\mathbf{x}^\top+\breve{\mathbf{R}}\big)^{-1}\\
    &\hat{\mathbf{x}}_{k+1|k+1} =     \hat{\mathbf{x}}_{k+1|k} + \mathbf{K}\breve{\mathbf{r}}\\
    &\mathbf{P}_{k+1|k+1} = \big(\mathbf{I}-\mathbf{K}\breve{\mathbf{H}}_\mathbf{x}\big)\mathbf{P}_{k+1|k}\big(\mathbf{I}-\mathbf{K}\breve{\mathbf{H}}_\mathbf{x}\big)^\top+\mathbf{K}\breve{\mathbf{R}}\mathbf{K}^\top.
\end{align*}

\subsubsection{\bf State augmentation}

To utilize the most accurate relative motion information for estimation, we employ the stochastic cloning~\cite{roumeliotis2002icra}. In particular, the state augmentation is performed right after the EKF update, where a copy of the updated relative pose estimate, $\{{^{k+1}_k}\hat{\bar{q}}_{k+1|k+1}$,${^{R_k}}\hat{\mathbf{p}}_{I_{k+1|k+1}}\}$, is appended to the end of the current sliding-window state, $\hat{\mathbf{w}}_{k+1|k+1}$. Accordingly, the covariance matrix is augmented as follows:
\begin{equation}
    \mathbf{P}_{k+1|k+1} \leftarrow
    \begin{bmatrix}
    \mathbf{I}_{24+6(N-1)} \\
    \mathbf{J}
    \end{bmatrix}
    \mathbf{P}_{k+1|k+1}
    \begin{bmatrix}
    \mathbf{I}_{24+6(N-1)} \\
    \mathbf{J}
    \end{bmatrix}^\top, \nonumber
\end{equation}
\begin{equation}
    \mathbf{J} =
    \begin{bmatrix}
    \mathbf{0}_{3\times9} & \mathbf{I}_3 & \mathbf{0}_3 & \mathbf{0}_{3\times9} & \mathbf{0}_{3\times6(N-1)} \\
    \mathbf{0}_{3\times9} & \mathbf{0}_3 & \mathbf{I}_3 & \mathbf{0}_{3\times9} & \mathbf{0}_{3\times6(N-1)}
    \end{bmatrix}
\end{equation}

\subsection{Composition}

Note that in the proposed robocentric formulation, every time when the update is finished, we shift the frame of reference of estimation. At this point, the IMU frame $\{I_{k+1}\}$, is set as the local frame of reference, i.e., $\{R_{k+1}\}$, to replace $\{R_k\}$. The state vector expressed in $\{R_{k+1}\}$ is then obtained as:
\begin{align}
    &\hat{\mathbf{x}}_{k+1} =
    \begin{bmatrix}
    {^{R_{k+1}}}\hat{\mathbf{x}}_{k+1} \\ \hat{\mathbf{w}}_{k+1}
    \end{bmatrix} =
    \begin{bmatrix}
    {^{R_k}}\hat{\mathbf{x}}_{k+1|k+1}\boxplus{^{R_k}}\hat{\mathbf{x}}_{I_{k+1|k+1}} \\
    \hat{\mathbf{w}}_{k+1|k+1}
    \end{bmatrix} \Rightarrow \nonumber \\
    &\begin{bmatrix}
    {^{k+1}_G}\hat{\bar{q}} \\ {^{R_{k+1}}}\hat{\mathbf{p}}_G \\ {^{R_{k+1}}}\hat{\mathbf{g}} \\ {^{k+1}_{k+1}}\hat{\bar{q}} \\ {^{R_{k+1}}}\hat{\mathbf{p}}_{R_{k+1}} \\ \hat{\mathbf{v}}_{R_{k+1}} \\ \hat{\mathbf{b}}_{g_{k+1}} \\ \hat{\mathbf{b}}_{a_{k+1}} \\
    \hat{\mathbf{w}}_{k+1}
    \end{bmatrix} =
    \begin{bmatrix}
    {^{k+1}_k}\hat{\bar{q}}\otimes{^k_G}\hat{\bar{q}} \\
    {^{k+1}_k}\mathbf{C}_{\hat{\bar{q}}}\big({^{R_k}}\hat{\mathbf{p}}_{G_{k+1}}-{^{R_k}}\hat{\mathbf{p}}_{I_{k+1}}\big) \\
    {^{k+1}_k}\mathbf{C}_{\hat{\bar{q}}}{^{R_k}}\hat{\mathbf{g}} \\
    \bar{q}_0 \\
    \mathbf{0}_{3\times1} \\
    \hat{\mathbf{v}}_{I_{k+1}} \\
    \hat{\mathbf{b}}_{g_{k+1}} \\
    \hat{\mathbf{b}}_{a_{k+1}} \\
    \hat{\mathbf{w}}_{k+1|k+1}
    \end{bmatrix} 
\end{align}
where $\bar{q}_0=[0,0,0,1]^\top$, $\boxplus$ denotes the state composition operator, and for brevity of presentation we have omitted the subscripts for the robocentric state. Note that, the relative pose in the IMU state is {\em reset} to the origin, while the velocity and biases in the current IMU frame are not affected by the change of frame of reference. The corresponding covariance composition is performed using the Jacobian:
\begin{equation}
    \mathbf{P}_{k+1} = \mathbf{U}_{k+1}\mathbf{P}_{k+1|k+1}\mathbf{U}_{k+1}^\top
    \label{eq:pcomp}
\end{equation}
\begin{equation}
    \begin{split}
    &\mathbf{U}_{k+1} = \frac{\partial{\tilde{\mathbf{x}}_{k+1}}}{\partial{\tilde{\mathbf{x}}_{k+1|k+1}}} =
    \begin{bmatrix}
    \mathbf{V}_{k+1} & \mathbf{0}_{24\times6N} \\
    \mathbf{0}_{6N\times24} & \mathbf{I}_{6N}
    \end{bmatrix}
    \end{split}
\end{equation}
where $\mathbf{V}_{k+1}$ is the Jacobian with respect to the robocentric state (see \eqref{eq:Jv}). Specifically, the corresponding covariance of the relative pose is also {\em reset} to zero, i.e., no uncertainty for the robocentric frame of reference itself.

\newcounter{tempEquationCounter1} 
\newcounter{thisEquationNumber1}
\newenvironment{floatEq1}
{\setcounter{thisEquationNumber1}{\value{equation}}\addtocounter{equation}{1}
\begin{figure*}[!t]
\normalsize\setcounter{tempEquationCounter1}{\value{equation}}
\setcounter{equation}{\value{thisEquationNumber1}}
}
{\setcounter{equation}{\value{tempEquationCounter1}}
\hrulefill\vspace*{1pt}
\end{figure*}
}
\begin{floatEq}
\begin{equation}
    \mathbf{V}_{k+1} =
    \frac{\partial{{^{R_{k+1}}}\tilde{\mathbf{x}}_{k+1}}}{\partial{{^{R_k}}\tilde{\mathbf{x}}_{k+1|k+1}}}
    = \begin{bmatrix}
    {^{k+1}_k}\mathbf{C}_{\hat{\bar{q}}} & \mathbf{0}_3 & \mathbf{0}_3 & \mathbf{I}_3 & \mathbf{0}_3 & \mathbf{0}_3 & \mathbf{0}_3 & \mathbf{0}_3 \\
    \mathbf{0}_3 & {^{k+1}_k}\mathbf{C}_{\hat{\bar{q}}} & \mathbf{0}_3 & \lfloor{^{R_{k+1}}}\hat{\mathbf{p}}_{G}\times\rfloor & -{^{k+1}_k}\mathbf{C}_{\hat{\bar{q}}} & \mathbf{0}_3 & \mathbf{0}_3 & \mathbf{0}_3 \\
    \mathbf{0}_3 & \mathbf{0}_3 & {^{k+1}_k}\mathbf{C}_{\hat{\bar{q}}} & \lfloor{^{R_{k+1}}}\hat{\mathbf{g}}\times\rfloor & \mathbf{0}_3 & \mathbf{0}_3 & \mathbf{0}_3 & \mathbf{0}_3 \\
    \mathbf{0}_3 & \mathbf{0}_3 & \mathbf{0}_3 & \mathbf{0}_3 & \mathbf{0}_3 & \mathbf{0}_3 & \mathbf{0}_3 & \mathbf{0}_3 \\
    \mathbf{0}_3 & \mathbf{0}_3 & \mathbf{0}_3 & \mathbf{0}_3 & \mathbf{0}_3 & \mathbf{0}_3 & \mathbf{0}_3 & \mathbf{0}_3 \\
    \mathbf{0}_3 & \mathbf{0}_3 & \mathbf{0}_3 & \mathbf{0}_3 & \mathbf{0}_3 & \mathbf{I}_3 & \mathbf{0}_3 & \mathbf{0}_3 \\
    \mathbf{0}_3 & \mathbf{0}_3 & \mathbf{0}_3 & \mathbf{0}_3 & \mathbf{0}_3 & \mathbf{0}_3 & \mathbf{I}_3 & \mathbf{0}_3 \\
    \mathbf{0}_3 & \mathbf{0}_3 & \mathbf{0}_3 & \mathbf{0}_3 & \mathbf{0}_3 & \mathbf{0}_3 & \mathbf{0}_3 & \mathbf{I}_3
    \end{bmatrix}_{k+1|k+1}
\label{eq:Jv}
\end{equation}
\end{floatEq}

\begin{algorithm} [!t]
\caption{Robocentric Visual-Inertial Odometry}
\begin{algorithmic}
\STATE \textbf{Input}: Camera images, and IMU measurements
\STATE \textbf{Output}: 6DOF real-time pose estimates
\STATE \textbf{R-VIO}: Initialize the state and covariance with respect to the first local frame of reference, $\{R_0\}$ (i.e., $\{G\}$), when the first available IMU measurement(s) comes in. Then, every time when a camera image is available, do
\begin{itemize}
\item \textbf{Visual tracking}: extract features from the image, then perform Kanade-Lucas-Tomasi (KLT) tracking and outlier rejection. Record the inliers' tracking histories within the current sliding window.
\item \textbf{Propagation}: propagate state and covariance matrix using preintegration with all the IMU measurements starting from last image time. \\
$\Rightarrow$\; $\mathbf{x}_k\rightarrow\mathbf{x}_{k+1|k}$, and $\mathbf{P}_k\rightarrow\mathbf{P}_{k+1|k}$.
\item \textbf{Update}: for the feature (inlier) whose track is complete (i.e., lost track, or reach the maximum tracking length), compute the inverse-depth measurement model matrices, then \\
-- {\em EKF update}: use the features that have passed the Mahalanobis distance test for an EKF update. \\
-- {\em State augmentation}: augment state vector and covariance matrix using the updated relative pose estimates (state and covariance). \\
$\Rightarrow$\; $\mathbf{x}_{k+1|k}\rightarrow\mathbf{x}_{k+1|k+1}$, and $\mathbf{P}_{k+1|k}\rightarrow\mathbf{P}_{k+1|k+1}$.
\item \textbf{Composition}: shift the frame of reference to current IMU frame, update global state and covariance using the updated relative pose estimates, then reset the relative pose (state and covariance). \\
$\Rightarrow$\; $\mathbf{x}_{k+1|k+1}\rightarrow\mathbf{x}_{k+1}$, and $\mathbf{P}_{k+1|k+1}\rightarrow\mathbf{P}_{k+1}$.
\end{itemize}
\end{algorithmic}
\label{ag:1}
\end{algorithm}

\subsection{Initialization}
\label{sec:init}

It is important to point out that in the proposed robocentric formulation, the filter initialization is very simple, because the states are simply relative to a local frame of reference and typically start from zero {\em without} the need to align the initial pose with a fixed global frame. In particular, in our implementation, (i) the initial global pose and IMU relative pose are both set to $\{\bar{q}_0,\mathbf{0}_{3\times1}\}$, (ii) the initial local gravity is the average of first available accelerometer measurement(s) before moving, and (iii) the initial value of acceleration bias is obtained by removing the gravity effects while the initial gyroscope bias is the average of the corresponding stationary measurements. Similarly, the corresponding uncertainties for the poses are set to zero, while for the local gravity and biases are set to be:
$\boldsymbol{\Sigma}_g=\Delta{T}\sigma_a^2\mathbf{I}_3$, $\boldsymbol{\Sigma}_{b_g}=\Delta{T}\sigma_{wg}^2\mathbf{I}_3$, and $\boldsymbol{\Sigma}_{b_a}=\Delta{T}\sigma_{wa}^2\mathbf{I}_3$, where $\Delta{T}$ is the time length of initialization. In summary, the main procedures of the proposed R-VIO are outlined in Algorithm~\ref{ag:1}.

\section{Observability analysis}

Observability of the system reveals whether the information provided by the measurements is sufficient to estimate the state without ambiguities. In this section, we examine the observability properties of the proposed R-VIO linearized system in the case of that a single landmark is observed by a mobile sensor platform performing arbitrary motions, while the conclusion of analysis can be generalized to the case of multiple landmarks. Note that, a direct analysis of the observability properties of R-VIO could be cumbersome due to the feature marginalization (see \eqref{eq:ol}), thus we perform the observability analysis using an EKF-SLAM model which has the same observability properties as an EKF-VIO model provided the same linearization points used, which has been shown as a common practice in the VINS literature (see \cite{li2013high,guo2013icra,hesch2014ijrr,hesch2014consistency}). 

To this end, the state vector at time-step $k$ includes a single landmark $L$:
\begin{equation}
    \mathbf{x}_k =
    \begin{bmatrix}
    {^{R_k}}\mathbf{x}_k^\top & {^{R_k}}\mathbf{p}_L^\top
    \end{bmatrix}^\top
    \label{eq:xobs}
\end{equation}
where ${^{R_k}}\mathbf{p}_L$ is the position of landmark with respect to the current local frame of reference, $\{R_k\}$. The measurement model~\eqref{eq:zm} (or the inverse-depth model \eqref{eq:zinv}) is used. The observability matrix is computed as~\cite{chen1990local}:
\begin{equation}
    \mathbf{M} =
    \begin{bmatrix}
    \mathbf{H}_k \\
    \vdots \\
    \mathbf{H}_\ell\boldsymbol{\Psi}_{\ell,k} \\
    \vdots \\
    \mathbf{H}_{k+m}\boldsymbol{\Psi}_{k+m,k}
    \end{bmatrix}
    \label{eq:obsM}
\end{equation}
where $\boldsymbol{\Psi}_{\ell,k}$ is the state transition matrix from time-step $k$ to $\ell$, and $\mathbf{H}_\ell$ is the measurement Jacobian corresponding to the observation(s) at time-step $\ell$. Each row is evaluated at ${^{R_k}}\hat{\mathbf{p}}_L$ and ${^{R_k}}\hat{\mathbf{x}}_i$, $i=k,\ldots,\ell,\ldots,k+m$. The nullspace of $\mathbf{M}$ describes the directions of the state space, in which no information is provided by the measurements, i.e., the unobservable state subspace. It should be noted that since the proposed robocentric EKF includes three steps: propagation, update, and composition, and the composition step changes the local frame of reference, we analyze the observability for a complete cycle of: (i) propagation and update, and (ii) composition. We analytically prove that the proposed R-VIO linearized system has a constant unobservable subspace, and dose {\em not} undergo the observability mismatch issue that has been shown to be the main cause of inconsistency~\cite{huang2010observability,li2013high,hesch2014ijrr,hesch2014consistency}, thus improving estimation performance.

\subsubsection{\bf Analytic error-state transition matrix}

For theoretical analysis, the analytic form error-state transition matrix is computed:
\begin{equation}
    \boldsymbol{\Psi}(\ell,k) =
    \begin{bmatrix}
    \boldsymbol{\Phi}(\ell,k) & \mathbf{0}_{24\times3} \\
    \mathbf{0}_{3\times24} & \mathbf{I}_3
    \end{bmatrix}
    \label{eq:psi}
\end{equation}
where, instead of \eqref{eq:phid}, $\boldsymbol{\Phi}(\ell,k)$ is obtained by integrating the following differential equation over the time interval $[t_k,t_\ell]$:
\begin{equation}
    \dot{\boldsymbol{\Phi}}(\ell,k) = \mathbf{F}\boldsymbol{\Phi}(\ell,k)
    \label{eq:phidot}
\end{equation}
with initial condition $\boldsymbol{\Phi}(k,k)=\mathbf{I}_{24}$. 
The closed form results can be found in the following, while the interested readers are referred to our companion technical report for detailed derivations~\cite{supp}:
\begin{align}
    &\boldsymbol{\Phi}(\ell,k) =
    \begin{bmatrix}
    \mathbf{I}_3& \mathbf{0}_3& \mathbf{0}_3& \mathbf{0}_3& \mathbf{0}_3& \mathbf{0}_3& \mathbf{0}_3& \mathbf{0}_3 \\
    \mathbf{0}_3& \mathbf{I}_3& \mathbf{0}_3& \mathbf{0}_3& \mathbf{0}_3& \mathbf{0}_3& \mathbf{0}_3& \mathbf{0}_3 \\
    \mathbf{0}_3& \mathbf{0}_3& \mathbf{I}_3& \mathbf{0}_3& \mathbf{0}_3& \mathbf{0}_3& \mathbf{0}_3& \mathbf{0}_3 \\
    \mathbf{0}_3& \mathbf{0}_3& \mathbf{0}_3&
    \boldsymbol{\Phi}_{44}& \mathbf{0}_3& \mathbf{0}_3& \boldsymbol{\Phi}_{47}& \mathbf{0}_3 \\
    \mathbf{0}_3& \mathbf{0}_3& \boldsymbol{\Phi}_{53}&
    \boldsymbol{\Phi}_{54}& \mathbf{I}_3& \boldsymbol{\Phi}_{56}& \boldsymbol{\Phi}_{57}& \boldsymbol{\Phi}_{58} \\
    \mathbf{0}_3& \mathbf{0}_3& \boldsymbol{\Phi}_{63}&
    \boldsymbol{\Phi}_{64}& \mathbf{0}_3& \boldsymbol{\Phi}_{66}& \boldsymbol{\Phi}_{67}& \boldsymbol{\Phi}_{68} \\
    \mathbf{0}_3& \mathbf{0}_3& \mathbf{0}_3& \mathbf{0}_3& \mathbf{0}_3& \mathbf{0}_3& \mathbf{I}_3& \mathbf{0}_3 \\
    \mathbf{0}_3& \mathbf{0}_3& \mathbf{0}_3& \mathbf{0}_3& \mathbf{0}_3& \mathbf{0}_3& \mathbf{0}_3& \mathbf{I}_3
    \end{bmatrix} \nonumber \\
    &\boldsymbol{\Phi}_{44}(\ell,k) =
    {^\ell_k}\mathbf{C}_{\hat{\bar{q}}} \label{eq:phi_first} \\
    &\boldsymbol{\Phi}_{47}(\ell,k) = -{^\ell_k}\mathbf{C}_{\hat{\bar{q}}}\int_{t_k}^{t_\ell}{^\tau_k}\mathbf{C}_{\hat{\bar{q}}}^\top\;{d\tau} \\
    &\boldsymbol{\Phi}_{53}(\ell,k) = -\frac{1}{2}\mathbf{I}_3\Delta{t}_{k,\ell}^2 \\
    &\boldsymbol{\Phi}_{54}(\ell,k) = -\lfloor\big({^{R_k}}\hat{\mathbf{p}}_{I_\ell}+\frac{1}{2}{^{R_k}}\hat{\mathbf{g}}\Delta{t}_{k,\ell}^2\big)\times\rfloor \\
    &\boldsymbol{\Phi}_{56}(\ell,k) = \mathbf{I}_3\Delta{t}_{k,\ell} \\
    &\boldsymbol{\Phi}_{57}(\ell,k) = \int_{t_k}^{t_\ell}\lfloor{^\tau_k}\mathbf{C}_{\hat{\bar{q}}}^\top\hat{\mathbf{v}}_{I_\tau}\times\rfloor\int_{t_k}^{\tau}{^\mu_k}\mathbf{C}_{\hat{\bar{q}}}^\top\;{d\mu}{d\tau} \nonumber\\
    &\qquad\qquad\quad+\lfloor{^{R_k}}\hat{\mathbf{g}}\times\rfloor\int_{t_k}^{t_\ell}\int_{t_k}^{\tau}\int_{t_k}^{\mu}{^\lambda_k}\mathbf{C}_{\hat{\bar{q}}}^\top\;{d\lambda}{d\mu}{d\tau} \nonumber \\
    &\qquad\qquad\quad-\int_{t_k}^{t_\ell}\int_{t_k}^{\tau}{^\mu_k}\mathbf{C}_{\hat{\bar{q}}}^\top\lfloor\hat{\mathbf{v}}_{I_\mu}\times\rfloor\;{d\mu}{d\tau} \\
    &\boldsymbol{\Phi}_{58}(\ell,k) = -\int_{t_k}^{t_\ell}\int_{t_k}^{\tau}{^\mu_k}\mathbf{C}_{\hat{\bar{q}}}^\top\;{d\mu}{d\tau} \\
    &\boldsymbol{\Phi}_{63}(\ell,k) = -{^\ell_k}\mathbf{C}_{\hat{\bar{q}}}\Delta{t}_{k,\ell} \\
    &\boldsymbol{\Phi}_{64}(\ell,k) = -{^\ell_k}\mathbf{C}_{\hat{\bar{q}}}\lfloor{^{R_k}}\hat{\mathbf{g}}\times\rfloor\Delta{t}_{k,\ell} \\
    &\boldsymbol{\Phi}_{66}(\ell,k) = \boldsymbol{\Phi}_{44}(\ell,k) \\
    &\boldsymbol{\Phi}_{67}(\ell,k) = {^\ell_k}\mathbf{C}_{\hat{\bar{q}}}\lfloor{^{R_k}}\hat{\mathbf{g}}\times\rfloor\int_{t_k}^{t_\ell}\int_{t_k}^{\tau}{^\mu_k}\mathbf{C}_{\hat{\bar{q}}}^\top\;{d\mu}{d\tau} \nonumber \\
    &\qquad\qquad\quad-\int_{t_k}^{t_\ell}{^\ell_\tau}\mathbf{C}_{\hat{\bar{q}}}\lfloor\hat{\mathbf{v}}_{I_\tau}\times\rfloor\;{d\tau} \\
    &\boldsymbol{\Phi}_{68}(\ell,k) = -{^\ell_k}\mathbf{C}_{\hat{\bar{q}}}\int_{t_k}^{t_\ell}{^\tau_k}\mathbf{C}_{\hat{\bar{q}}}^\top\;{d\tau}
    \label{eq:phi_end}
\end{align}
where $\Delta{t}_{k,\ell}=t_\ell-t_k$.

\subsubsection{\bf Measurement Jacobian} At time-step $\ell\in[t_k,t_{k+m}]$, the position estimate of landmark in $\{I_\ell\}$ can be expressed as:
\begin{equation}
    {^{I_\ell}}\hat{\mathbf{p}}_L = {^\ell_k}\mathbf{C}_{\hat{\bar{q}}}\big({^{R_k}}\hat{\mathbf{p}}_L-{^{R_k}}\hat{\mathbf{p}}_{I_\ell}\big)
\end{equation}
Based on \eqref{eq:zm}, the bearing-only measurement is given by:
\begin{equation}
    \mathbf{z}_\ell = \frac{1}{z}
    \begin{bmatrix}
    x \\ y
    \end{bmatrix}, \quad
    {^{I_\ell}}\mathbf{p}_L =
    \begin{bmatrix}
    x & y & z
    \end{bmatrix}^\top
    \label{eq:zl}
\end{equation}
Notice that for brevity of presentation, here we assume that the camera and IMU frames coincide. The corresponding measurement Jacobian is in the form:
\begin{align}
    \mathbf{H}_\ell &=
    \mathbf{H}_\text{p}{^\ell_k}\mathbf{C}_{\hat{\bar{q}}}
    \begin{bmatrix}
    \mathbf{0}_3 & \mathbf{0}_3 & \mathbf{0}_3 & \mathbf{H}_{\boldsymbol{\theta}_\ell} & -\mathbf{I}_3 & \mathbf{0}_{3\times9} \;\; \big\lvert \;\; \mathbf{I}_3
    \end{bmatrix} \nonumber \\
    \mathbf{H}_\text{p} &= \frac{1}{\hat{z}}
    \begin{bmatrix}
    1 & 0 & -\frac{\hat{x}}{\hat{z}} \\
    0 & 1 & -\frac{\hat{y}}{\hat{z}}
    \end{bmatrix}, \;\; 
    \mathbf{H}_{\boldsymbol{\theta}_\ell} = \lfloor\big({^{R_k}}\hat{\mathbf{p}}_L-{^{R_k}}\hat{\mathbf{p}}_{I_\ell}\big)\times\rfloor{^\ell_k}\mathbf{C}_{\hat{\bar{q}}}^\top
    \label{eq:Jobs}
\end{align}

\subsection{Observability of propagation and update}

Based on the above equations, we obtain the $\ell$-th block row, $\mathbf{M}_\ell$, of $\mathbf{M}$, as follows (see \eqref{eq:psi}, \eqref{eq:phi_first}-\eqref{eq:phi_end}, and \eqref{eq:Jobs}):
\begin{align}
    \mathbf{M}_\ell &= \mathbf{H}_\ell\boldsymbol{\Psi}_{\ell,k} \nonumber \\
    &=\boldsymbol{\Pi}
    \begin{bmatrix}
    \mathbf{0}_3 & \mathbf{0}_3 & \boldsymbol{\Gamma}_1 & \boldsymbol{\Gamma}_2 & -\mathbf{I}_3 & \boldsymbol{\Gamma}_3 & \boldsymbol{\Gamma}_4 & \boldsymbol{\Gamma}_5 \;\; \big\lvert \;\; \mathbf{I}_3
    \end{bmatrix} \nonumber
\end{align}
where
\begin{align}
    \boldsymbol{\Pi} &= \mathbf{H}_\text{p}{^\ell_k}\mathbf{C}_{\hat{\bar{q}}} \\
    \boldsymbol{\Gamma}_1 &= -\boldsymbol{\Phi}_{53} = \frac{1}{2}\mathbf{I}_3\Delta{t}_{k,\ell}^2 \\
    \boldsymbol{\Gamma}_2 &= \lfloor\big({^{R_k}}\hat{\mathbf{p}}_L-{^{R_k}}\hat{\mathbf{p}}_{I_\ell}\big)\times\rfloor{^\ell_k}\mathbf{C}_{\hat{\bar{q}}}^\top\boldsymbol{\Phi}_{44}-\boldsymbol{\Phi}_{54} \nonumber \\
    &= \lfloor{^{R_k}}\hat{\mathbf{p}}_L\times\rfloor+\frac{1}{2}\lfloor{^{R_k}}\hat{\mathbf{g}}\times\rfloor\Delta{t}_{k,\ell}^2 \\
    \boldsymbol{\Gamma}_3 &= -\boldsymbol{\Phi}_{56} = -\mathbf{I}_3\Delta{t}_{k,\ell} \\
    \boldsymbol{\Gamma}_4 &= \lfloor\big({^{R_k}}\hat{\mathbf{p}}_L-{^{R_k}}\hat{\mathbf{p}}_{I_\ell}\big)\times\rfloor{^\ell_k}\mathbf{C}_{\hat{\bar{q}}}^\top\boldsymbol{\Phi}_{47}-\boldsymbol{\Phi}_{57} \nonumber \\
    &= -\lfloor\big({^{R_k}}\hat{\mathbf{p}}_L-{^{R_k}}\hat{\mathbf{p}}_{I_\ell}\big)\times\rfloor\int_{t_k}^{t_\ell}{^\tau_k}\mathbf{C}_{\hat{\bar{q}}}^\top\;{d\tau}-\boldsymbol{\Phi}_{57} \\
    \boldsymbol{\Gamma}_5 &= -\boldsymbol{\Phi}_{58}
\end{align}
Note that for generic motion, i.e., $\boldsymbol{\omega}\neq\mathbf{0}_{3\times1}$ and $\mathbf{a}\neq\mathbf{0}_{3\times1}$, the values of $\boldsymbol{\Phi}_{57}$ and $\boldsymbol{\Phi}_{58}$ are time-varying, then $\boldsymbol{\Gamma}_4$ and $\boldsymbol{\Gamma}_5$ are linearly independent. Moreover, the value of $\Delta{t}_{k,\ell}$ is varying for different time intervals, then the stacked $\boldsymbol{\Gamma}_1$, $\boldsymbol{\Gamma}_2$, and $\boldsymbol{\Gamma}_3$ are linearly independent. Thus, the stacked $\boldsymbol{\Gamma}_1$, $\boldsymbol{\Gamma}_2$, $\boldsymbol{\Gamma}_3$, $\boldsymbol{\Gamma}_4$, and $\boldsymbol{\Gamma}_5$ are linearly independent. Based on that, we perform Gaussian elimination on $\mathbf{M}_\ell$ to facilitate the search for the nullspace:
\begin{align}
    \mathbf{M}_\ell &= \boldsymbol{\Pi}
    \begin{bmatrix}
    \mathbf{0}_3 & \mathbf{0}_3 & \boldsymbol{\Gamma}_1 & \boldsymbol{\Gamma}_2 & -\mathbf{I}_3 & \boldsymbol{\Gamma}_3 & \boldsymbol{\Gamma}_4 & \boldsymbol{\Gamma}_5 \;\; \big\lvert \;\; \mathbf{I}_3
    \end{bmatrix} \nonumber \\
    &\thicksim \boldsymbol{\Pi}
    \begin{bmatrix}
    \mathbf{0}_3 & \mathbf{0}_3 & \boldsymbol{\Gamma}_1 & \boldsymbol{\Gamma}_2 & -\mathbf{I}_3 & \boldsymbol{\Gamma}_3 & \boldsymbol{\Gamma}_4 & \boldsymbol{\Gamma}_5 \;\; \big\lvert \;\; \mathbf{0}_3
    \end{bmatrix} \nonumber
\end{align}
from which we can find that $\mathbf{M}_\ell$ is rank deficient by $9$, and accordingly the nullspace is of rank $9$. Specifically, $\forall{\ell}\geq{k}$, we can find that the nullspace of $\mathbf{M}$ consists of the following {\em nine} directions, as:
\begin{equation}
    \textbf{null}(\mathbf{M}) = \underset{\operatorname{col.}}{\operatorname{span}}
    \begin{bmatrix}
    \mathbf{I}_3 & \mathbf{0}_3 & \mathbf{0}_3 \\
    \mathbf{0}_3 & \mathbf{I}_3 & \mathbf{0}_3 \\
    \mathbf{0}_3 & \mathbf{0}_3 & \mathbf{0}_3 \\
    \mathbf{0}_3 & \mathbf{0}_3 & \mathbf{0}_3 \\
    \mathbf{0}_3 & \mathbf{0}_3 & \mathbf{I}_3 \\
    \mathbf{0}_3 & \mathbf{0}_3 & \mathbf{0}_3 \\
    \mathbf{0}_3 & \mathbf{0}_3 & \mathbf{0}_3 \\
    \mathbf{0}_3 & \mathbf{0}_3 & \mathbf{0}_3 \\
    \mathbf{0}_3 & \mathbf{0}_3 & \mathbf{I}_3
    \end{bmatrix}
    \label{eq:nullM}
\end{equation}
which may be interpreted as follows:
\begin{remark}
    The first 6 DOF correspond to the orientation (3) and position (3) of the global frame, while the last 3 DOF belong to the same translation (3) simultaneously applied to the sensor and landmark(s). This agrees with our intuition that relative IMU and camera measurements do not provide any global state information, which is analogous to the SLAM case~\cite{huang2010observability}.
\end{remark}

\subsection{Observability with composition}

After update at time-step $\ell$, the estimates of ${^{R_k}}\mathbf{x}_\ell$ and ${^{R_k}}\mathbf{p}_L$ are obtained, we have the following linear model from time-step $k$ to $\ell$, including the composition step, as:
\begin{equation}
    \tilde{\mathbf{x}}_\ell = \check{\mathbf{V}}_\ell\boldsymbol{\Psi}(\ell,k)\tilde{\mathbf{x}}_k = \check{\boldsymbol{\Psi}}(\ell,k)\tilde{\mathbf{x}}_k
\end{equation}
where
\begin{align}
    \check{\mathbf{V}}_\ell &=
    \begin{bmatrix}
    \mathbf{V}_\ell & \mathbf{0}_{24\times3} \\
    \mathbf{L}_\ell & \mathbf{N}_\ell
    \end{bmatrix} =
    \begin{bmatrix}
    \mathbf{V}_\ell & \mathbf{0}_{24\times3} \\
    \frac{\partial{\tilde{\mathbf{x}}_\ell}}{\partial{{^{R_k}}\tilde{\mathbf{x}}_\ell}} & \frac{\partial{\tilde{\mathbf{x}}_\ell}}{\partial{{^{R_k}}\tilde{\mathbf{p}}_L}}
    \end{bmatrix} \nonumber \\
    \mathbf{L}_\ell &=
    \begin{bmatrix}
    \mathbf{0}_3 & \mathbf{0}_3 & \mathbf{0}_3 &  \lfloor{^{R_\ell}}\hat{\mathbf{p}}_{L}\times\rfloor & -{^\ell_k}\mathbf{C}_{\hat{\bar{q}}} & \mathbf{0}_3 & \mathbf{0}_3 & \mathbf{0}_3
    \end{bmatrix}, \nonumber \\
    \mathbf{N}_\ell &= {^\ell_k}\mathbf{C}_{\hat{\bar{q}}}
\end{align}
For brevity of analysis, only the pertinent entries of $\check{\boldsymbol{\Psi}}(\ell,k)$ (see \eqref{eq:psicomp}) are shown in the following:
\newcounter{tempEquationCounter2} 
\newcounter{thisEquationNumber2}
\newenvironment{floatEq2}
{\setcounter{thisEquationNumber2}{\value{equation}}\addtocounter{equation}{1}
\begin{figure*}[!h]
\normalsize\setcounter{tempEquationCounter2}{\value{equation}}
\setcounter{equation}{\value{thisEquationNumber2}}
}
{\setcounter{equation}{\value{tempEquationCounter2}}
\hrulefill\vspace*{1pt}
\end{figure*}
}
\begin{floatEq}
\begin{equation}
    \check{\boldsymbol{\Psi}}(\ell,k) = \check{\mathbf{V}}_\ell\boldsymbol{\Psi}(\ell,k) =
    \begin{bmatrix}
    \check{\boldsymbol{\Psi}}_{11} & \mathbf{0}_3 & \mathbf{0}_3 & \check{\boldsymbol{\Psi}}_{14} & \mathbf{0}_3 & \mathbf{0}_3 & \check{\boldsymbol{\Psi}}_{17} & \mathbf{0}_3 & \mathbf{0}_3 \\
    \mathbf{0}_3 & \check{\boldsymbol{\Psi}}_{22} & \check{\boldsymbol{\Psi}}_{23} & \check{\boldsymbol{\Psi}}_{24} & \check{\boldsymbol{\Psi}}_{25} & \check{\boldsymbol{\Psi}}_{26} & \check{\boldsymbol{\Psi}}_{27} & \check{\boldsymbol{\Psi}}_{28} & \mathbf{0}_3 \\
    \mathbf{0}_3 & \mathbf{0}_3 & \check{\boldsymbol{\Psi}}_{33} & \check{\boldsymbol{\Psi}}_{34} & \mathbf{0}_3 & \mathbf{0}_3 & \check{\boldsymbol{\Psi}}_{37} & \mathbf{0}_3 & \mathbf{0}_3 \\
    \mathbf{0}_3 & \mathbf{0}_3 & \mathbf{0}_3 &
    \mathbf{0}_3 & \mathbf{0}_3 & \mathbf{0}_3 &
    \mathbf{0}_3 & \mathbf{0}_3 & \mathbf{0}_3 \\
    \mathbf{0}_3 & \mathbf{0}_3 & \mathbf{0}_3 &
    \mathbf{0}_3 & \mathbf{0}_3 & \mathbf{0}_3 &
    \mathbf{0}_3 & \mathbf{0}_3 & \mathbf{0}_3 \\
    \mathbf{0}_3 & \mathbf{0}_3 & \check{\boldsymbol{\Psi}}_{63} & \check{\boldsymbol{\Psi}}_{64} & \mathbf{0}_3 &
    \check{\boldsymbol{\Psi}}_{66} & \check{\boldsymbol{\Psi}}_{67} & \check{\boldsymbol{\Psi}}_{68} & \mathbf{0}_3 \\
    \mathbf{0}_3 & \mathbf{0}_3 & \mathbf{0}_3 & \mathbf{0}_3 & \mathbf{0}_3 &
    \mathbf{0}_3 & \mathbf{I}_3 & \mathbf{0}_3 & \mathbf{0}_3 \\
    \mathbf{0}_3 & \mathbf{0}_3 & \mathbf{0}_3 & \mathbf{0}_3 & \mathbf{0}_3 &
    \mathbf{0}_3 & \mathbf{0}_3 & \mathbf{I}_3 & \mathbf{0}_3 \\
    \mathbf{0}_3 & \mathbf{0}_3 & \check{\boldsymbol{\Psi}}_{93} & \check{\boldsymbol{\Psi}}_{94} & \check{\boldsymbol{\Psi}}_{95} & \check{\boldsymbol{\Psi}}_{96} & \check{\boldsymbol{\Psi}}_{97} & \check{\boldsymbol{\Psi}}_{98} & \check{\boldsymbol{\Psi}}_{99}
    \end{bmatrix}
    \label{eq:psicomp}
\end{equation}
\end{floatEq}
\begin{align}
    &\check{\boldsymbol{\Psi}}_{93} = -{^\ell_k}\mathbf{C}_{\hat{\bar{q}}}\boldsymbol{\Phi}_{53} \\
    &\check{\boldsymbol{\Psi}}_{94} = \lfloor{^{R_\ell}}\hat{\mathbf{p}}_L\times\rfloor\boldsymbol{\Phi}_{44}-{^\ell_k}\mathbf{C}_{\hat{\bar{q}}}\boldsymbol{\Phi}_{54} \\
    &\check{\boldsymbol{\Psi}}_{95} = -{^\ell_k}\mathbf{C}_{\hat{\bar{q}}} \\
    &\check{\boldsymbol{\Psi}}_{96} = -{^\ell_k}\mathbf{C}_{\hat{\bar{q}}}\boldsymbol{\Phi}_{56} \\
    &\check{\boldsymbol{\Psi}}_{97} = \lfloor{^{R_\ell}}\hat{\mathbf{p}}_L\times\rfloor\boldsymbol{\Phi}_{47}-{^\ell_k}\mathbf{C}_{\hat{\bar{q}}}\boldsymbol{\Phi}_{57} \\
    &\check{\boldsymbol{\Psi}}_{98} = -{^\ell_k}\mathbf{C}_{\hat{\bar{q}}}\boldsymbol{\Phi}_{58} \\
    &\check{\boldsymbol{\Psi}}_{99} = {^\ell_k}\mathbf{C}_{\hat{\bar{q}}}
\end{align}
Note that the measurement model of \eqref{eq:zl} becomes linear:
\begin{equation}
    \mathbf{z}_\ell = {^{R_\ell}}\mathbf{p}_L, \quad
    {^{R_\ell}}\mathbf{p}_L = {^\ell_k}\mathbf{C}_{\bar{q}}\big({^{R_k}}\mathbf{p}_L-{^{R_k}}\mathbf{p}_{I_\ell}\big)
\end{equation}
and the measurement Jacobian with respect to $\tilde{\mathbf{x}}_\ell$ is as:
\begin{equation}
    \check{\mathbf{H}}_\ell =
    \begin{bmatrix}
    \mathbf{0}_{3\times24} \;\; \big\lvert \;\; \mathbf{I}_3
    \end{bmatrix}
\end{equation}
Therefore, after composition we have the block row, $\mathbf{M}_\ell$, of $\mathbf{M}$ in the form of:
\begin{align}
    \mathbf{M}_\ell &=
    \check{\mathbf{H}}_\ell\check{\boldsymbol{\Psi}}_{\ell,k} \nonumber \\
    &=
    \begin{bmatrix}
    \mathbf{0}_3 & \mathbf{0}_3 & \check{\boldsymbol{\Psi}}_{93:94} & -{^\ell_k}\mathbf{C}_{\hat{\bar{q}}} & \check{\boldsymbol{\Psi}}_{96:98} \;\; \big\lvert \;\; {^\ell_k}\mathbf{C}_{\hat{\bar{q}}}
    \end{bmatrix} \nonumber
\end{align}
where for generic motion case, i.e., $\boldsymbol{\omega}\neq\mathbf{0}_{3\times1}$ and $\mathbf{a}\neq\mathbf{0}_{3\times1}$, $\check{\boldsymbol{\Psi}}_{93}$, $\check{\boldsymbol{\Psi}}_{94}$, $\check{\boldsymbol{\Psi}}_{96}$, $\check{\boldsymbol{\Psi}}_{97}$, and $\check{\boldsymbol{\Psi}}_{98}$ are linearly independent, and obviously the same nullspace as that of the propagation and update can be obtained (see \eqref{eq:nullM}).

\begin{remark}
    In the proposed robocentric model, changing local frame of reference by composition does not alter the unobservable subspace.
\end{remark}

Thus far, we have shown that the proposed robocentric model has a {\em constant} unobservable subspace, i.e., independent of the linearization points. This not only guarantees that the system has correct unobservable dimensions as~\cite{huang2010observability,li2013high,hesch2014ijrr,hesch2014consistency}, but also the desired unobservable directions, thus being expected to improve estimation consistency.

\subsection{Observability under special motions}

Depending on the motion undertaken, the system observability properties might change in some degenerate cases. Identifying and understanding such special motions is essential for improving the VINS performance, especially in practice. The most commonly seen case is the planar motion (where usually the translation is only excited in the $x$-$y$ plane, and the rotation is only about the $z$-axis) and the recent analysis on world-centric VINS~\cite{wu2017vins} has pointed out that in this type of motion two more unobservable directions emerge: (i) the global orientation, and (ii) the scale. Note that, for the proposed robocentric VINS model the global orientation has already been shown to be unobservable (see \eqref{eq:nullM}), thus, in what follows we study in-depth the observability under special motions by focusing on the scale (un)observability.

\subsubsection{\bf Effect of scaling on VINS states}

We are first to understand the implications of an underlying scale factor applied to the state vector of the proposed robocentric system, which will form the basis for identifying the degenerate motions causing the special unobservable directions.

\begin{lem} \label{lem:obs1}
For the proposed robocentric system, given the true state, $\mathbf{x}$, and the underlying state, $\mathbf{x}'$, that are related through a scale factor, $s$, there exists the following relation between the corresponding error states (see \eqref{eq:xobs}):
\begin{equation}
    \begin{bmatrix}
    \delta\boldsymbol{\theta}_G \\
    {^{R_k}}\tilde{\mathbf{p}}_G \\
    {^{R_k}}\tilde{\mathbf{g}} \\
    \delta\boldsymbol{\theta}_I \\
    {^{R_k}}\tilde{\mathbf{p}}_I \\
    \tilde{\mathbf{v}}_I \\
    \tilde{\mathbf{b}}_g \\
    \tilde{\mathbf{b}}_a \\
    {^{R_k}}\tilde{\mathbf{p}}_L
    \end{bmatrix} =
    \begin{bmatrix}
    \delta\boldsymbol{\theta}'_G \\
    {^{R_k}}\tilde{\mathbf{p}}'_G \\
    {^{R_k}}\tilde{\mathbf{g}}' \\
    \delta\boldsymbol{\theta}'_I \\
    {^{R_k}}\tilde{\mathbf{p}}'_I \\
    \tilde{\mathbf{v}}'_I \\
    \tilde{\mathbf{b}}'_g \\
    \tilde{\mathbf{b}}'_a \\
    {^{R_k}}\tilde{\mathbf{p}}'_L
    \end{bmatrix} + (s-1)
    \begin{bmatrix}
    \mathbf{0}_{3\times1} \\
    {^{R_k}}\tilde{\mathbf{p}}'_G \\
    \mathbf{0}_{3\times1} \\
    \mathbf{0}_{3\times1} \\
    {^{R_k}}\tilde{\mathbf{p}}'_I \\
    \tilde{\mathbf{v}}'_I \\
    \mathbf{0}_{3\times1} \\
    -{^I}\mathbf{a}' \\
    {^{R_k}}\tilde{\mathbf{p}}'_L
    \end{bmatrix} \Rightarrow \nonumber
\end{equation}
\begin{equation}
    \tilde{\mathbf{x}} = \tilde{\mathbf{x}}'+(s-1)\mathbf{u}
    \label{eq:sx}
\end{equation}
\end{lem}

\begin{proof}
    \textup{See Appendix B}.
\end{proof}

\subsubsection{\bf Special motions for scale unobservability} 

It becomes clear from \eqref{eq:sx} that if the proposed robocentric VINS estimation is metrically scaled by a factor of $s$, then the error state (and hence the state) would be changed along the direction of $\mathbf{u}$ by a factor of $(s-1)$. However, as evident from the proof (see Appendix B), 
we cannot distinguish this scale ambiguity from the camera and IMU measurements, which implies that the direction of scale is {\em unobservable}. The following analysis further identifies the special motions that can cause this scale unobservability. 

\begin{lem}
For the proposed robocentric system, there exist two special motions which can cause scale unobservable: (i) no rotations, with:
\begin{equation}
    \Delta{t}_{k,\ell}\hat{\mathbf{v}}'_{I_\ell} = -\frac{1}{2}\Delta{t}_{k,\ell}^2{^\ell}\hat{\mathbf{a}}', \quad \forall \ell\geq k
    \label{eq:ucond}
\end{equation}
and (ii) constant local acceleration, with $\hat{\mathbf{v}}'_{I_\tau}={^\tau}\hat{\mathbf{a}}'\equiv\mathbf{0}$, $\forall \tau\in[t_k,t_\ell]$; that is, the system is stationary. 
\end{lem}

\begin{proof}
    \textup{See Appendix C}.
\end{proof}

As a final remark, it is clear from the above lemma that the scale unobservable direction does exist when: (i) \eqref{eq:ucond} holds (e.g., during the deceleration phase), or (ii) the sensor platform remains stationary. However, these two cases can be easily mitigated in practice. Specifically, in the case of  (i), as it holds true as $\Delta{t}_{k,\ell}\rightarrow0$, we can simply increase $\Delta{t}_{k,\ell}$ in practice to avoid the scale change. While in the case of (ii),  we will confront low parallax, but the inverse-depth measurement model used in the proposed R-VIO (see \eqref{eq:zinv}) will enable it only to exploit the rotation information from the measurements, thus holding the scale. It should be pointed out that, in contrast to the world-centric remedy~\cite{wu2017vins} where the wheel odometry measurements are fused, the proposed R-VIO does not need an additional sensor to address this scale issue, thus revealing the better adaptability and robustness.

\section{Simulation results}

In this section, we present Monte Carlo simulation results that verify the analysis provided in the preceding sections and illustrate the performance of the proposed R-VIO algorithm compared to two world-centric counterparts: (i) the standard (Std)-MSCKF~\cite{mourikis2007multi}, and (ii) the state-of-the-art state-transition observability constrained (STOC)-MSCKF~\cite{huang2014towards} that enforces correct observability to improve consistency. In particular, two metrics are used for evaluation: (a) the root mean squared error (RMSE) that provides a concise metric of the filter's accuracy, and (b) the normalized estimation error squared (NEES) which offers a standard criterion for evaluating the given filter's consistency~\cite{Bar-Shalom2001}. In order to make a fair comparison, we implemented all filters using the same parameters, such as the sliding-window size, and processing the same data in all 50 Monte Carlo trails that are generated at real MEMS sensor noise and bias levels (see Figure \ref{fig:sim-scene}).

The statistical results over 50 Monte Carlo trails are shown in Figure~\ref{fig:sim-rmse-nees}, and Table~\ref{tab:sim-rmse-nees} provides the average RMSE and NEES results for all the algorithms compared in this test, which clearly show that the proposed R-VIO significantly outperforms the standard MSCKF and the STOC-MSCKF in terms of both RMSE (accuracy) and NEES (consistency), attributed to the novel reformulation of the system. Note that, in Figure~\ref{fig:sim-rmse-nees} the orientation NEES of R-VIO has a jump at the beginning which is primarily due to the small covariance we used for initialization, while it can quickly recover and perform consistently only after a short period of time.

\begin{figure}[!t]
    \centering
    \includegraphics[width=.48\textwidth]{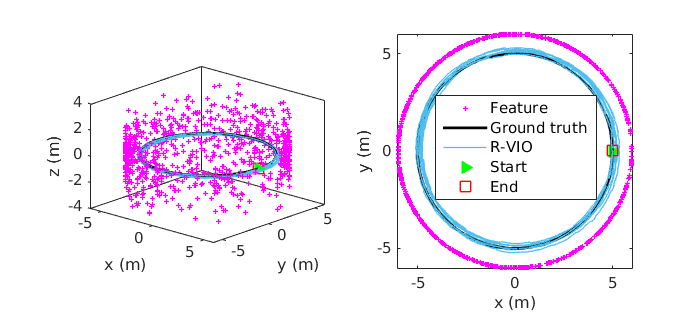}
    \caption{Simulation scenario: a camera/IMU pair moves along a circular path of radius 5m ({\bf black}) at an average speed of 1m/s. The camera with $45^\circ$ field of view observes point features ({\bf pink}) randomly distributed on a circumscribing cylinder of radius 6m. The standard deviation of image noise is set to 1.5 pixels. The IMU provides 3DOF angular velocities and linear accelerations which are generated with actual MEMS sensor's quality, i.e., $\sigma_g=1.122e^{-4}\sf{rad}/\sf{sec}/\sqrt{\sf{Hz}}$, $\sigma_{wg}=5.6323e^{-6}\sf{rad}/\sf{sec}^2/\sqrt{\sf{Hz}}$, $\sigma_a=5.0119e^{-4}\sf{m}/\sf{sec}^2/\sqrt{\sf{Hz}}$, and $\sigma_{wa}=3.9811e^{-5}\sf{m}/\sf{sec}^3/\sqrt{\sf{Hz}}$.}
    \label{fig:sim-scene}
\end{figure}

\begin{figure}
    \centering
    \includegraphics[width=.49\textwidth]{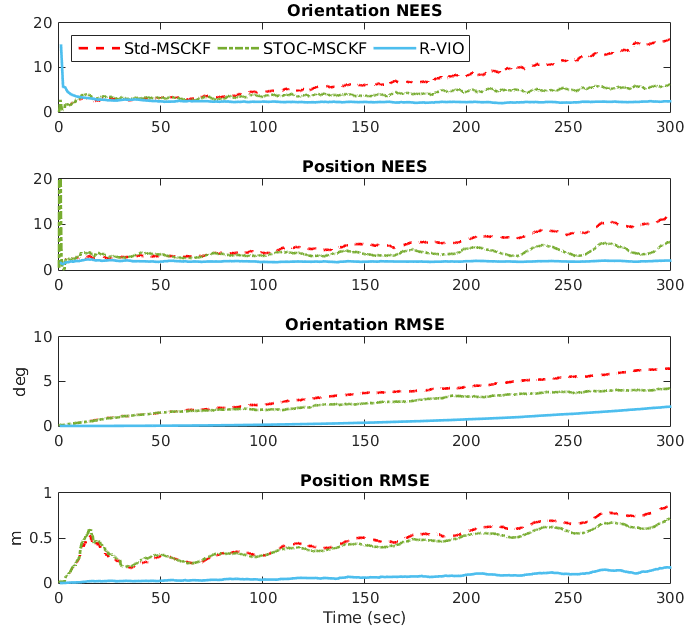}
    \caption{Simulation results: the average NEES and RMSE of orientation and position over 50 Monte-Carlo trials.}
    \label{fig:sim-rmse-nees}
\end{figure}

\begin{table}
\centering
\caption{Avg. RMSE and NEES corresponding to Fig. \ref{fig:sim-rmse-nees}.}
\begin{tabular}{|l|c|c|c|c|}
\hline
& Orien. & Pos. & Orien. & Pos. \\
& RMSE (deg) & RMSE (m) & NEES & NEES \\
\hline
Std-MSCKF & 3.470 & 0.477 & 7.048 & 5.810 \\
STOC-MSCKF & 2.523 & 0.430 & 4.096 & 3.793 \\
R-VIO & {\bf 0.681} & {\bf 0.071} & {\bf 2.414} & {\bf 1.906} \\
\hline
\end{tabular}
\label{tab:sim-rmse-nees}
\end{table}

\section{Experimental results}
\label{sec:expm}

We further experimentally validate the proposed R-VIO in both indoor and outdoor environments, using both the public benchmark dataset on micro aerial vehicle (MAV) 
and the data collected with our own sensor platforms, including the hand-held and urban driving datasets. As described in Algorithm~\ref{ag:1}, we implemented it with C++ multithread framework. In the {\em front end}, the visual tracking thread extracts features from the image using the Shi-Tomasi corner detector~\cite{shi1994good}, and tracks them between pairwise images using the Kanade-Lucas-Tomasi (KLT) algorithm~\cite{baker2004lucas}. In particular, to deal with the varying lighting conditions in practice, a preprocessing of Gaussian thresholding and box blurring was applied for each image before doing the KLT tracking. This effectively mitigates the sharp change of illumination and outlines the structures of environment even in the dark areas (see Figure~\ref{fig:frontend}), which is particularly helpful for the feature detection. In addition, to remove the outliers from the visual tracks, we realized the gyro-aided two-point RANSAC algorithm~\cite{troiani20142}. In the end, all the inliers' tracking histories are stored in a first-in-first-out (FIFO) data structure which can be efficiently queried during the estimation.

Once the visual tracking is done, the {\em back end} processes all the visual and inertial measurements using the proposed robocentric EKF. Especially, for the feature lost track we use all its measurements within the sliding window for an EKF update, while for the one reaching the maximum tracking length (e.g., the sliding-window size) we use its subset (e.g., 1/2) of measurements and maintain the rest for next update. 
All the tests run on a Core i7-4710MQ @ 2.5GHz laptop at {\em real time}.

\begin{figure}
    \centering
    \begin{subfigure}{0.48\textwidth}
        \centering
        \includegraphics[width=\textwidth]{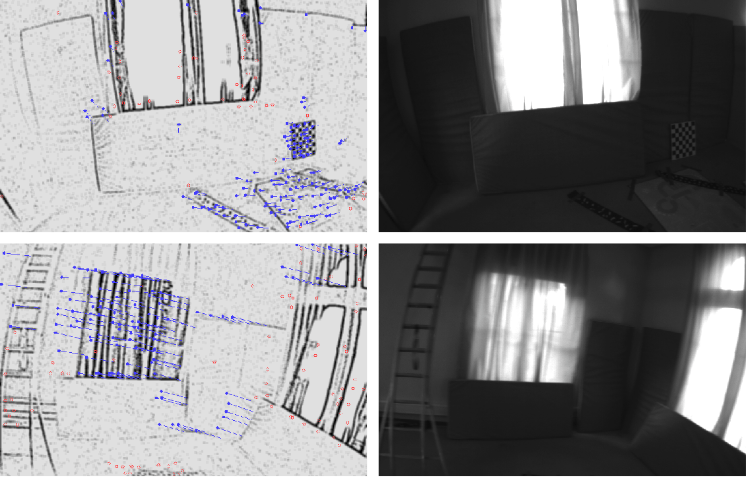}
        \caption{EuRoC dataset ({Vicon room}): V1\_03\_difficult.}
    \end{subfigure}
    \begin{subfigure}{0.48\textwidth}
        \centering
        \vspace{1pc}
        \includegraphics[width=\textwidth]{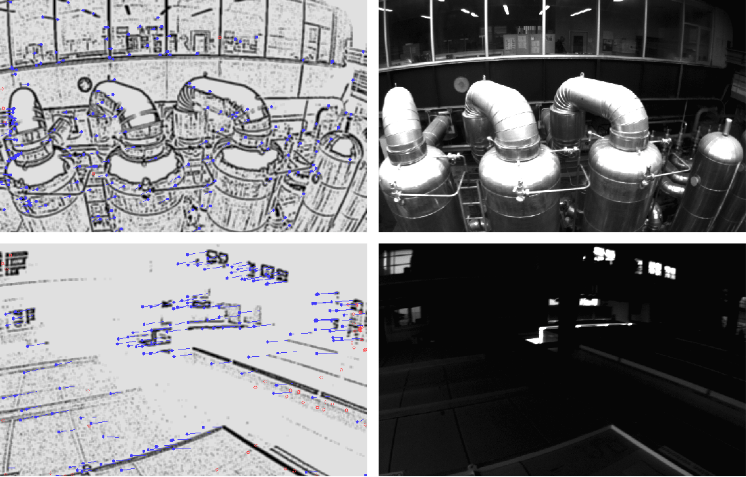}
        \caption{EuRoC dataset ({Machine hall}): MH\_05\_difficult.}
    \end{subfigure}
    \begin{subfigure}{0.48\textwidth}
        \centering
        \vspace{1pc}
        \includegraphics[width=\textwidth]{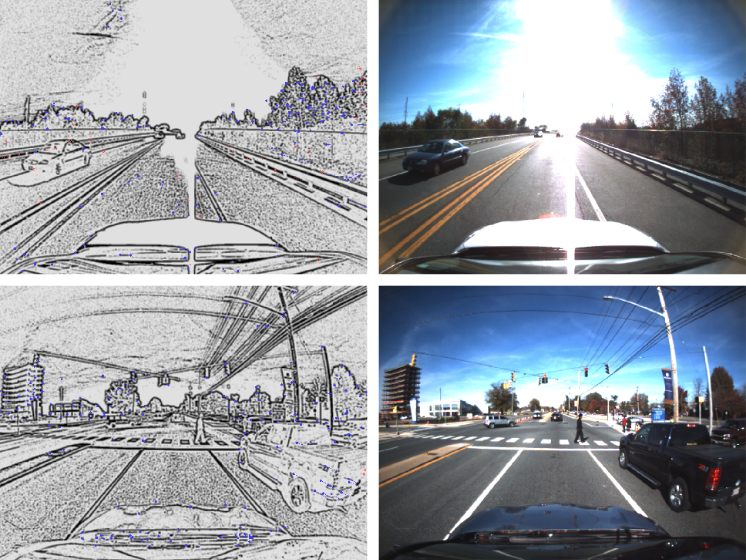}
        \caption{Urban Driving dataset.}
    \end{subfigure}
    \caption{Visual tracking: the processing results (left column) and the corresponding raw images (right column). The inliers ({\bf blue}) are tracked between pairwise images with the outliers ({\bf red}) being rejected by 2-point RANSAC. The performance of outlier rejection can be illustrated by the tracks of inliers (the blue lines) showing the trend of camera motion. It is important to note that the proposed visual tracking method is able to handle the (a) blurred, (b) dark, and (c) overexposed scenes of the real world.}
    \label{fig:frontend}
\end{figure}

\begin{figure}
    \centering
    \begin{subfigure}{0.235\textwidth}
        \centering
        \includegraphics[width=\textwidth,height=3.3cm]{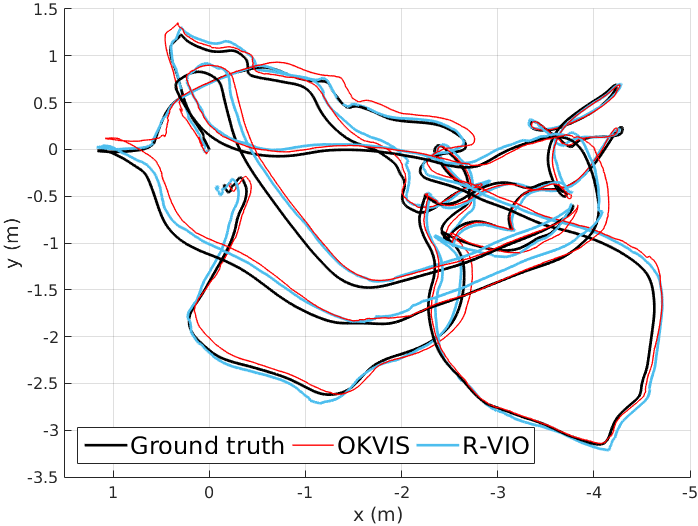}
        \caption{V1\_01\_easy}
    \end{subfigure}
    \begin{subfigure}{0.235\textwidth}
        \centering
        \includegraphics[width=\textwidth,height=3.3cm]{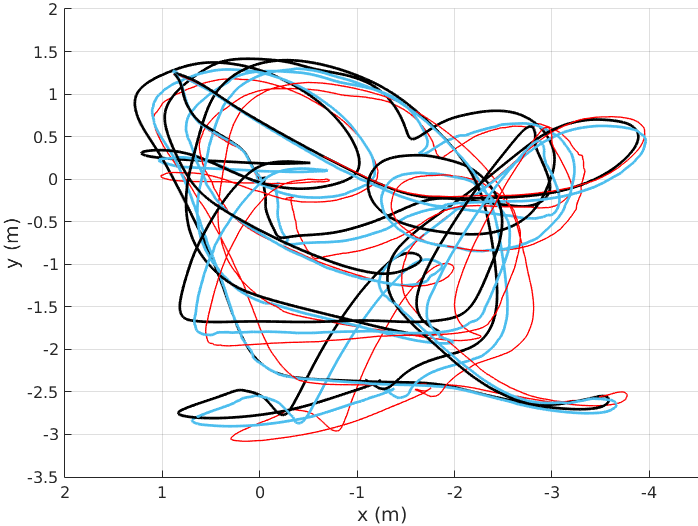}
        \caption{V1\_02\_medium}
    \end{subfigure}
    \begin{subfigure}{0.235\textwidth}
        \centering
        \vspace{0.5pc}
        \includegraphics[width=\textwidth,height=3.3cm]{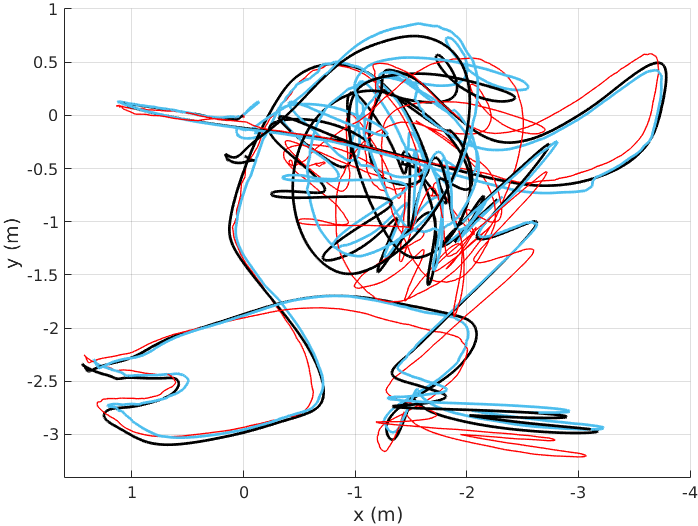}
        \caption{V1\_03\_difficult}
    \end{subfigure}
    \begin{subfigure}{0.235\textwidth}
        \centering
        \vspace{0.5pc}
        \includegraphics[width=\textwidth,height=3.3cm]{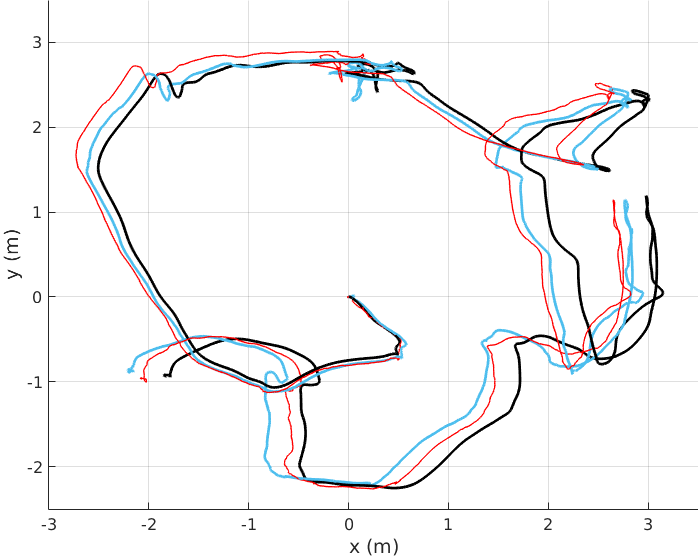}
        \caption{V2\_01\_easy}
    \end{subfigure}
    \begin{subfigure}{0.235\textwidth}
        \centering
        \vspace{0.5pc}
        \includegraphics[width=\textwidth,height=3.3cm]{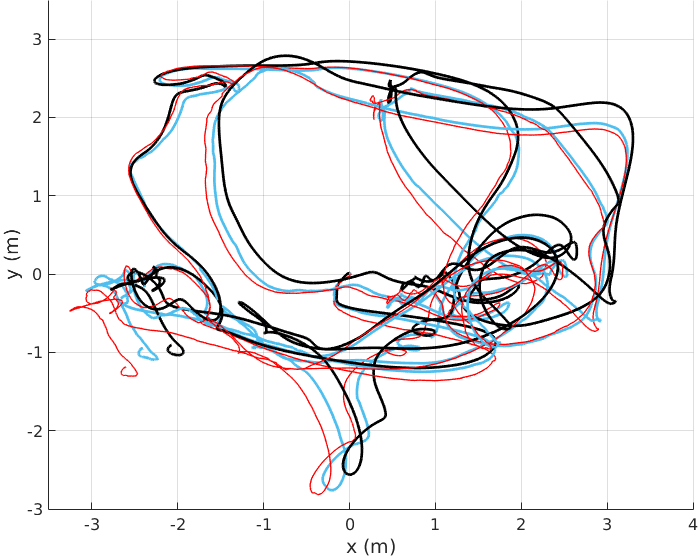}
        \caption{V2\_02\_medium}
    \end{subfigure}
    \begin{subfigure}{0.235\textwidth}
        \centering
        \vspace{0.5pc}
        \includegraphics[width=\textwidth,height=3.3cm]{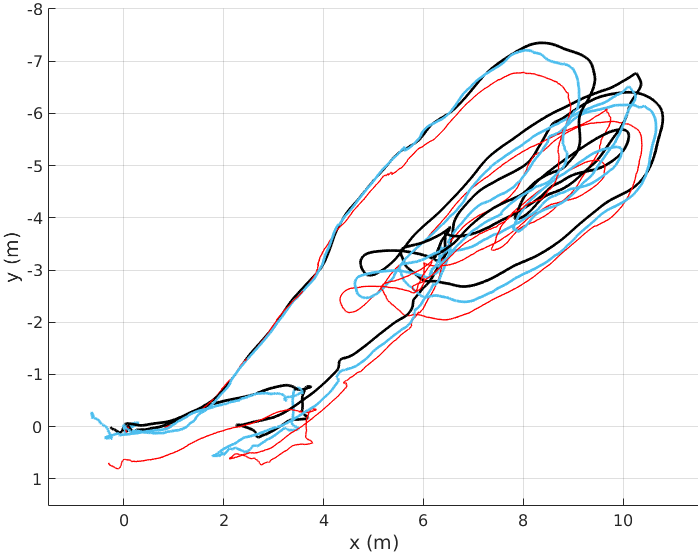}
        \caption{MH\_01\_easy}
    \end{subfigure}
    \begin{subfigure}{0.235\textwidth}
        \centering
        \vspace{0.5pc}
        \includegraphics[width=\textwidth,height=3.3cm]{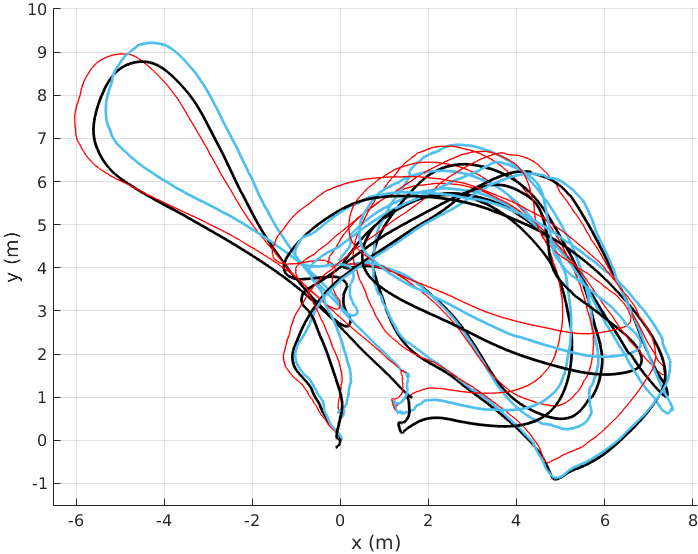}
        \caption{MH\_03\_medium}
    \end{subfigure}
    \begin{subfigure}{0.235\textwidth}
        \centering
        \vspace{0.5pc}
        \includegraphics[width=\textwidth,height=3.3cm]{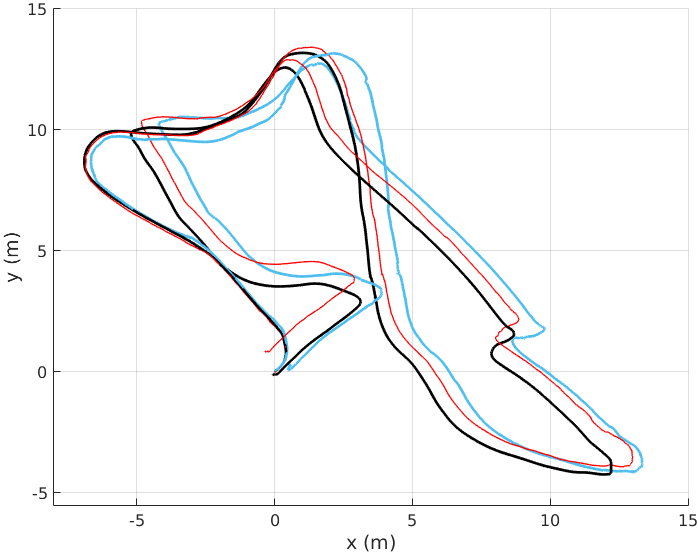}
        \caption{MH\_05\_difficult}
    \end{subfigure}
    \caption{Trajectory estimates in EuRoC dataset.}
    \label{fig:eurocres}
\end{figure}

\begin{table}[!b]
\centering
\caption{Estimation accuracy (RMSE) in EuRoC dataset.}
\begin{tabular}{|l|c|c|c|c|c|}
\cline{1-6}
& & \multicolumn{2}{|c|}{OKVIS} & \multicolumn{2}{|c|}{R-VIO} \\
\cline{2-6}
& Length & Orien. & Pos. & Orien. & Pos. \\
& (m) & (deg) & (m) & (deg) & (m) \\
\hline
\multicolumn{1}{|l|}{V1\_01\_easy} & 58.6 & 2.350 & 0.142 & {2.151} & {\bf 0.085} \\
\multicolumn{1}{|l|}{V1\_02\_medium} & 75.9 & 3.363 & 0.299 & {0.777} & {\bf 0.156} \\
\multicolumn{1}{|l|}{V1\_03\_difficult} & 79.0 & 3.586 & 0.265 & {0.729} & {\bf 0.137} \\
\multicolumn{1}{|l|}{V2\_01\_easy} & 36.5 & {0.651} & 0.311 & 1.014 & {\bf 0.216} \\
\multicolumn{1}{|l|}{V2\_02\_medium} & 83.2 & 2.986 & 0.341 & {1.214} & {\bf 0.313} \\
\multicolumn{1}{|l|}{V2\_03\_difficult} & 86.1 & 5.912 & {\bf 0.377} & {1.275} & 0.441 \\
\hline
\multicolumn{1}{|l|}{MH\_01\_easy} & 80.6 & {1.051} & 0.590 & 1.236 & {\bf 0.387} \\
\multicolumn{1}{|l|}{MH\_02\_easy} & 73.5 & 1.062 & {\bf 0.698} & {0.946} & 0.740 \\
\multicolumn{1}{|l|}{MH\_03\_medium} & 130.9 & 2.336 & 0.550 & {1.351} & {\bf 0.358} \\
\multicolumn{1}{|l|}{MH\_04\_difficult} & 91.7 & {0.286} & {\bf 0.431} & 3.525 & 1.037 \\
\multicolumn{1}{|l|}{MH\_05\_difficult} & 97.6 & {1.136} & {\bf 0.674} & 1.392 & 0.858 \\
\hline
\end{tabular}
\label{tab:euroc-rmse}
\end{table}

\subsection{EuRoC dataset}

We tested the proposed R-VIO on all of 11 sequences in EuRoC dataset~\cite{burri2016euroc}, in which a FireFly hex-rotor helicopter equipped with VI-sensor (an IMU @ 200Hz and dual cameras 752$\times$480 pixels @ 20Hz) was used for data collection. In this test, only the left camera images were used for vision inputs, and 200 features were uniformly extracted from each image. The sliding-window size was set up to 20 (i.e., about 1 second memory of the relative motion). We compared the proposed R-VIO against the OKVIS\footnote{https://github.com/ethz-asl/okvis}, one state-of-the-art world-centric keyframe-based visual-inertial SLAM system~\cite{leutenegger2015keyframe} performing nonlinear iterative optimization for estimation. The RMSE results after 6DOF pose alignment are shown in Table~\ref{tab:euroc-rmse}, and Figure~\ref{fig:eurocres} depicts the estimated trajectories in 8 representative sequences. It is important to note that the proposed R-VIO does not utilize any kind of map, while the OKVIS does. Nevertheless, in general, the R-VIO performs comparably to the OKVIS, and even {\em better} in most sequences (see Table~\ref{tab:euroc-rmse}).

\subsection{Hand-held dataset}

We also validated the proposed R-VIO both indoor and outdoor with one of our own sensor platforms (a MicroStrain 3DM-GX3-35 IMU @ 500Hz and a PointGrey Chameleon3 monocular camera 644$\times$482 pixels @ 30Hz) that was rigidly mounted onto the laptop. Both daytime and nighttime data were collected for the indoor test, where we travelled 150m at an average speed of 0.539m/s, covering two floors in a building (with white walls, variant illumination, and strong glare in the hallway, see Figure~\ref{fig:hand-indoor1}), then coming back to the start point; while the outdoor test used the data of a 360m loop recorded at an average speed of 1.216m/s (with uneven terrain and opportunistic moving objects, see Figure~\ref{fig:hand-outdoor1}). Due to the lack of the ground truth, here in order to illustrate the performance we overlay the estimated trajectories onto the floor plan and the map, respectively (see Figure~\ref{fig:hand-indoor2} and \ref{fig:hand-outdoor2}). The final position errors are 0.349\% (daytime) and 0.615\% (nighttime) over the distance travelled in the indoor test, and 1.173\% in the outdoor test.

\begin{figure}
    \centering
    \begin{subfigure}{.48\textwidth}
        \centering
        \includegraphics[width=\textwidth]{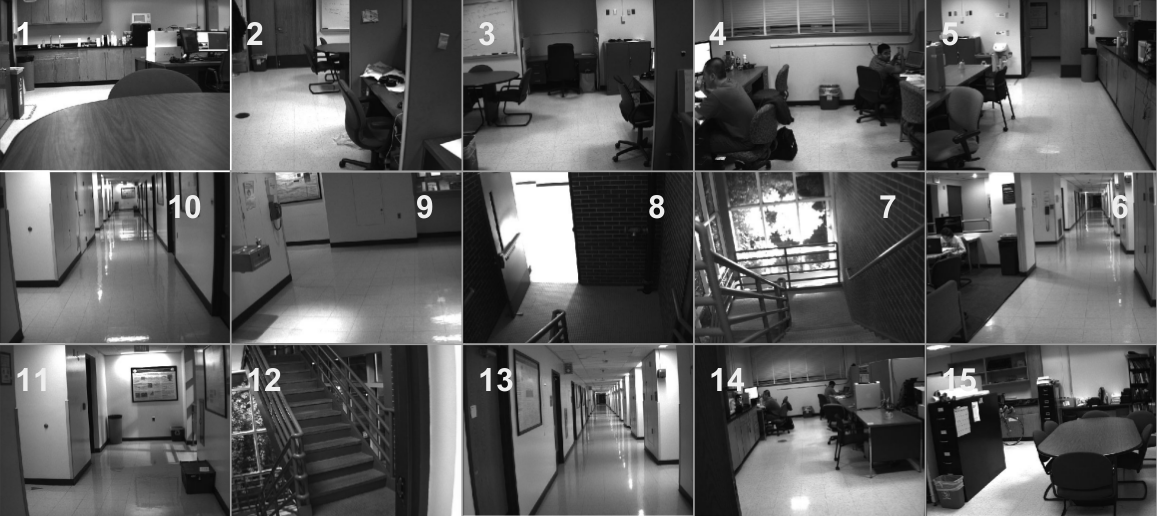}
        \caption{Snapshots during the indoor test (nighttime).}
        \label{fig:hand-indoor1}
    \end{subfigure}
    \begin{subfigure}{.455\textwidth}
        \centering
        \vspace{.5pc}
        \includegraphics[width=\textwidth]{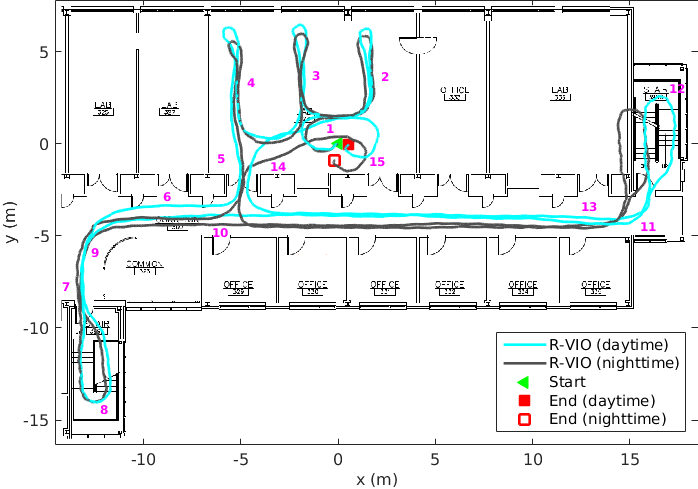}
        \caption{Indoor trajectory plotted over the floor plan.}
        \label{fig:hand-indoor2}
    \end{subfigure}
    \begin{subfigure}{0.48\textwidth}
        \centering
        \vspace{.5pc}
        \includegraphics[width=\textwidth]{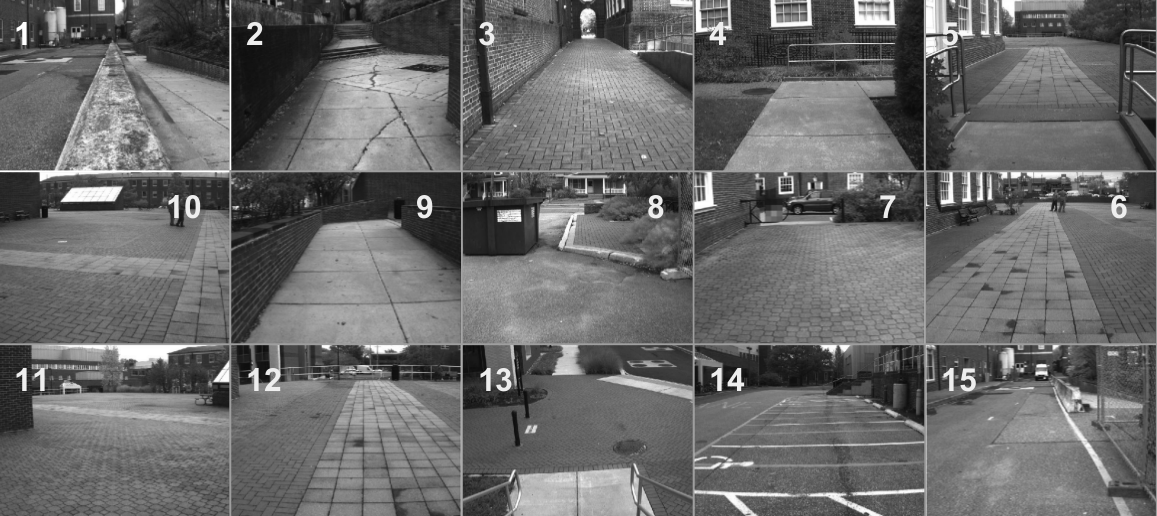}
        \caption{Snapshots during the outdoor test (daytime).}
        \label{fig:hand-outdoor1}
    \end{subfigure}
    \begin{subfigure}{0.455\textwidth}
        \centering
        \vspace{.5pc}
        \includegraphics[width=\textwidth]{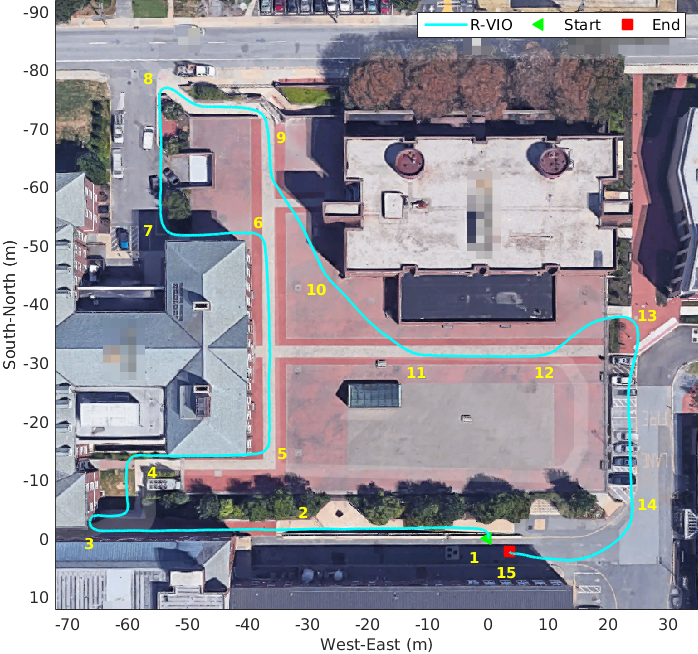}
        \caption{Outdoor trajectory plotted over a map.}
        \label{fig:hand-outdoor2}
    \end{subfigure}
    \caption{Results of Hand-held dataset.}
    \label{fig:handexpm}
\end{figure}

\subsection{Urban Driving dataset}

We further performed a road test using a car equipped with another sensor platform (an Xsens Mti-G INS/GNSS and a FLIR Bumblebee2 stereo pair 1024$\times$768 pixels @ 15Hz), and driving on the streets of Newark, DE. The IMU provided measurements at 400Hz, while the GPS signal was received at 4Hz as the (position) ground truth. 
Similarly, only the left camera images were used for vision inputs, with 200 features being uniformly extracted from each image. 
It is important to point out that the test is challenging primarily due to: 
(i) several traffic lights at which we must stop and wait for 15-25 seconds, 
(ii) frequent stop/yield signs before which we must decelerate or stop, 
(iii) dynamic scenes including the running vehicles and the pedestrians in vicinity, 
(iv) strong lens flare when driving facing the sun, and 
(v) high speeds of vehicle when driving in some areas (see Figure~\ref{fig:veh-scene}).
Because of these, the OKVIS was not able to provide reasonable localization results while the proposed R-VIO still performed well during the test.

\begin{figure*}
    \centering
    \includegraphics[width=0.7\textwidth]{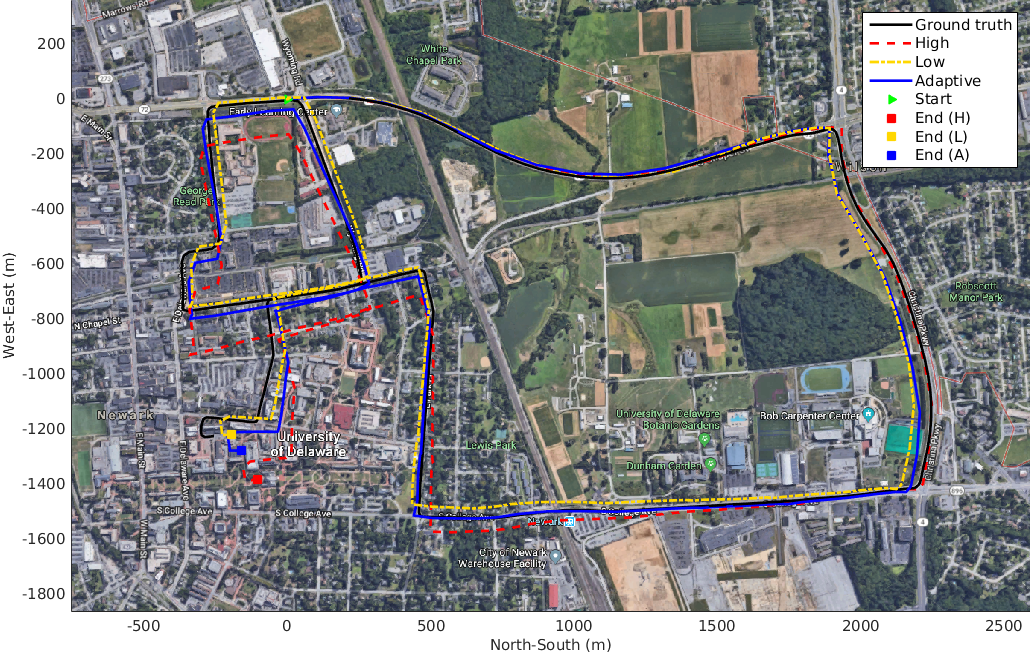}
    \caption{Trajectory estimates plotted over a map of Newark, DE. The initial position of vehicle is marked by a green triangle. The black solid line corresponds to the ground truth (GPS), the red dashed line to the result of {high} update rate, the yellow dash-dotted line to the result of {low} update rate, and the blue solid line to the result of {adaptive} update rate, with the end positions marked by the squares in the corresponding colors, respectively.}
    \label{fig:veh-traj}
\end{figure*}

\begin{figure}[!h]
    \centering
    \includegraphics[width=.48\textwidth]{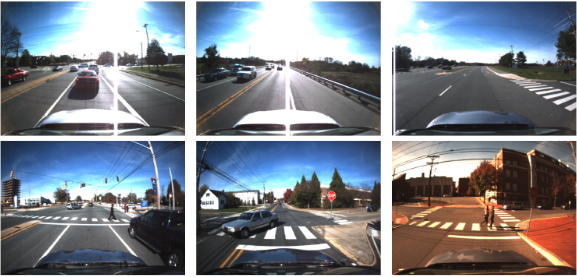}
    \caption{Snapshots during the urban driving test.}
    \label{fig:veh-scene}
\end{figure}

As what we discussed, both (i) and (ii) are the degenerate scenes which make the scale unobservable for the proposed VINS model. The usage of inverse-depth based measurement mode (see \eqref{eq:zinv}) solved the scale drift during the static phase, while for the deceleration phase we tested three update rates: high (15Hz), low (7Hz), and adaptive (switching between high and low). In particular, for the adaptive mode the R-VIO lowered down the update rate once recognizing deceleration phase from the changes of speed. The results are summarized in Table~\ref{tab:veh-rmse}, and Figure~\ref{fig:veh-traj} shows the estimated trajectories for all three update rates. We can find that using high update rate R-VIO captures high dynamic motion better than using low update rate, for instance, after the first right turn the vehicle sped up to 86km/h where the trajectory under high update rate fitted the ground truth better. While at the second right turn, a series of decelerations occurred due to the busy traffic at the intersection, as a consequence the scale issue biased the estimated trajectory afterwards. In contrast to that, with low update rate the R-VIO compensated the scale drift which makes entire trajectory closer to the ground truth. As a result, the proposed adaptive scheme is to take both the aforementioned advantages. Those performances are further confirmed by a test for which the difference of translation between consecutive poses of the estimates, $\Delta_{est}$, and that of the ground truth, $\Delta_{gt}$, are compared for every 10 seconds. The results referring to the estimated speeds are presented in Figure~\ref{fig:veh-scale}, from which we can find that the large differences (e.g., $>$5m) only appear when the sharp decelerations occur, while after the static phases the differences become much smaller. Among the three cases, the adaptive one performs the best with the average drift of 5.917m, while 8.274m and 5.992m for the high and low update rates, respectively.

\begin{figure}
    \centering
    \includegraphics[width=.49\textwidth]{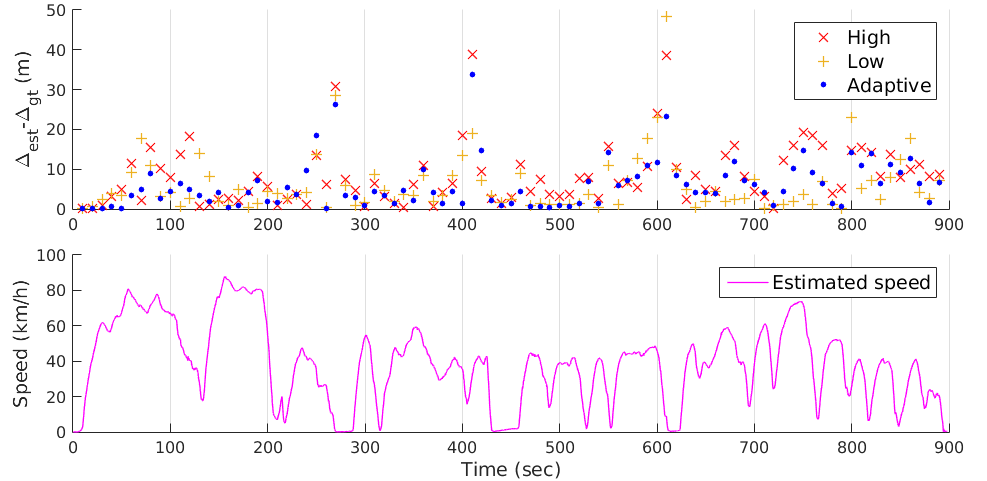}
    \caption{Relative translation error results vs. speeds of: high ({\bf red}), low ({\bf yellow}), and adaptive ({\bf blue}) update rates.}
    \label{fig:veh-scale}
\end{figure}

\begin{table}[!h]
\centering
\caption{Estimation accuracy (RMSE) in Urban Driving dataset of: high ($\sharp$1), low ($\sharp$2), and adaptive ($\sharp$3) update rates.}
\begin{tabular}{|l|c|c|c|c|c|}
\cline{1-6}
& Length / & \multicolumn{1}{|c|}{Max. speed} & \multicolumn{3}{|c|}{Avg. Position RMSE} \\
\cline{4-6}
& Duration & (km/h) & $x$ (m) & $y$ (m) & $z$ (m) \\
\hline
\multicolumn{1}{|l|}{$\sharp$1} & 9.8km / 15min & 85.9 & 30.934 & 68.561 & 8.418 \\
\hline
\multicolumn{1}{|l|}{$\sharp$2} & - / - & - & 33.984 & 15.883 & 10.426 \\
\hline
\multicolumn{1}{|l|}{$\sharp$3} & - / - & - & {\bf 24.222} & {\bf 18.901} & {\bf 7.689} \\
\hline
\end{tabular}
\label{tab:veh-rmse}
\end{table}

Note that, as the local gravity is jointly estimated, the $z$-axis drifts are much smaller than the $x$-$y$ position errors. The sliding-window size 20 was used in the test, and the average processing time of pipeline is 59.3 milliseconds per frame, including the 54.8 milliseconds spent on the visual tracking and feature management, and the other 4.5 millisecond on the robocentric EKF. 
For this challenging driving scenario, without using any kind of map, the proposed R-VIO achieves the average position RMSEs of: 0.77\% (high update rate), 0.40\% (low update rate), and 0.32\% (adaptive update rate) of the total distance travelled.


\section{Conclusion and future work}

In this paper, we have reformulated the VINS with respect to a moving local frame and developed a lightweight, high-precision, robocentric visual-inertial odometry algorithm, termed R-VIO. With this novel reformulation, we analytically show that with generic motion, the resulting VINS does not suffer from the observability mismatch issue encountered in the world-centric counterparts, and even in the degenerate motion case (planar motion) the observability issue can be easily compensated without using additional sensor information, thus offering better consistency, accuracy and robustness. 
Extensive Monte Carlo simulations and the real-world experiments using different sensor platforms and navigating in different environments were performed to thoroughly validate our theoretical analysis and show that the proposed R-VIO is versatile and robust to different types of motions and environments, and is capable of providing long-term, high-precision 3D motion tracking in real time.
In the future, we will integrate efficient loop closure and online mapping into the current robocentric system in order to bound localization errors,
as well as perform online calibration of  intrinsic and extrinsic sensor parameters to further improve  performance.

\section{Acknowledgement}
This work was partially supported by the University of Delaware College of Engineering, UD Cybersecurity Initiative, the Delaware NASA/EPSCoR Seed Grant, the NSF (IIS-1566129), and the DTRA (HDTRA1-16-1-0039). The authors would also like to thank Patrick Geneva for helping collect the urban driving data.

\bibliographystyle{IEEEtran}  
\bibliography{ref}

\begin{thebibliography}{10}
\providecommand{\url}[1]{#1}
\csname url@samestyle\endcsname
\providecommand{\newblock}{\relax}
\providecommand{\bibinfo}[2]{#2}
\providecommand{\BIBentrySTDinterwordspacing}{\spaceskip=0pt\relax}
\providecommand{\BIBentryALTinterwordstretchfactor}{4}
\providecommand{\BIBentryALTinterwordspacing}{\spaceskip=\fontdimen2\font plus
\BIBentryALTinterwordstretchfactor\fontdimen3\font minus
  \fontdimen4\font\relax}
\providecommand{\BIBforeignlanguage}[2]{{%
\expandafter\ifx\csname l@#1\endcsname\relax
\typeout{** WARNING: IEEEtran.bst: No hyphenation pattern has been}%
\typeout{** loaded for the language `#1'. Using the pattern for}%
\typeout{** the default language instead.}%
\else
\language=\csname l@#1\endcsname
\fi
#2}}
\providecommand{\BIBdecl}{\relax}
\BIBdecl

\bibitem{mourikis2007multi}
A.~I. Mourikis and S.~I. Roumeliotis, ``A multi-state constraint kalman filter
  for vision-aided inertial navigation,'' in \emph{IEEE International
  Conference on Robotics and Automation}, Rome, Italy, April 2007, pp.
  3565--3572.

\bibitem{jones2011visual}
E.~S. Jones and S.~Soatto, ``Visual-inertial navigation, mapping and
  localization: A scalable real-time causal approach,'' \emph{The International
  Journal of Robotics Research}, vol.~30, no.~4, pp. 407--430, 2011.

\bibitem{kelly2011visual}
J.~Kelly and G.~S. Sukhatme, ``Visual-inertial sensor fusion: Localization,
  mapping and sensor-to-sensor self-calibration,'' \emph{The International
  Journal of Robotics Research}, vol.~30, no.~1, pp. 56--79, 2011.

\bibitem{li2013high}
M.~Li and A.~I. Mourikis, ``High-precision, consistent ekf-based
  visual-inertial odometry,'' \emph{The International Journal of Robotics
  Research}, vol.~32, no.~6, pp. 690--711, 2013.

\bibitem{leutenegger2015keyframe}
S.~Leutenegger, S.~Lynen, M.~Bosse, R.~Siegwart, and P.~Furgale,
  ``Keyframe-based visual--inertial odometry using nonlinear optimization,''
  \emph{The International Journal of Robotics Research}, vol.~34, no.~3, pp.
  314--334, 2015.

\bibitem{shen2015tightly}
S.~Shen, N.~Michael, and V.~Kumar, ``Tightly-coupled monocular visual-inertial
  fusion for autonomous flight of rotorcraft mavs,'' in \emph{IEEE
  International Conference on Robotics and Automation}, Seattle, WA, May 2015,
  pp. 5303--5310.

\bibitem{usenko2016direct}
V.~Usenko, J.~Engel, J.~St{\"u}ckler, and D.~Cremers, ``Direct visual-inertial
  odometry with stereo cameras,'' in \emph{IEEE International Conference on
  Robotics and Automation}, Stockholm, Sweden, May 2016, pp. 1885--1892.

\bibitem{mur2017visual}
R.~Mur-Artal and J.~D. Tard{\'o}s, ``Visual-inertial monocular slam with map
  reuse,'' \emph{IEEE Robotics and Automation Letters}, vol.~2, no.~2, pp.
  796--803, 2017.

\bibitem{huang2010observability}
G.~P. Huang, A.~I. Mourikis, and S.~I. Roumeliotis, ``Observability-based rules
  for designing consistent ekf slam estimators,'' \emph{The International
  Journal of Robotics Research}, vol.~29, no.~5, pp. 502--528, 2010.

\bibitem{martinelli2012vision}
A.~Martinelli, ``Vision and imu data fusion: Closed-form solutions for
  attitude, speed, absolute scale, and bias determination,'' \emph{IEEE
  transactions on robotics}, vol.~28, no.~1, pp. 44--60, 2012.

\bibitem{hesch2014consistency}
J.~A. Hesch, D.~G. Kottas, S.~L. Bowman, and S.~I. Roumeliotis, ``Consistency
  analysis and improvement of vision-aided inertial navigation,'' \emph{IEEE
  transactions on robotics}, vol.~30, no.~1, pp. 158--176, 2014.

\bibitem{huang2014towards}
G.~Huang, M.~Kaess, and J.~J. Leonard, ``Towards consistent visual-inertial
  navigation,'' in \emph{IEEE International Conference on Robotics and
  Automation}, Hong Kong, China, May 2014, pp. 4926--4933.

\bibitem{zhang2017convergence}
T.~Zhang, K.~Wu, J.~Song, S.~Huang, and G.~Dissanayake, ``Convergence and
  consistency analysis for a 3-d invariant-ekf slam,'' \emph{IEEE Robotics and
  Automation Letters}, vol.~2, no.~2, pp. 733--740, 2017.

\bibitem{castellanos2004limits}
J.~A. Castellanos, J.~Neira, and J.~D. Tard{\'o}s, ``Limits to the consistency
  of ekf-based slam,'' \emph{IFAC Proceedings Volumes}, vol.~37, no.~8, pp.
  716--721, 2004.

\bibitem{civera20091p}
J.~Civera, O.~G. Grasa, A.~J. Davison, and J.~Montiel, ``1-point ransac for
  ekf-based structure from motion,'' in \emph{IEEE/RSJ International Conference
  on Intelligent Robots and Systems}, St. Louis, MO, October 2009, pp.
  3498--3504.

\bibitem{bloesch2015robust}
M.~Bloesch, S.~Omari, M.~Hutter, and R.~Siegwart, ``Robust visual inertial
  odometry using a direct ekf-based approach,'' in \emph{IEEE/RSJ International
  Conference on Intelligent Robots and Systems}.\hskip 1em plus 0.5em minus
  0.4em\relax Hamburg, Germany: IEEE, September 2015, pp. 298--304.

\bibitem{bloesch2017iterated}
M.~Bloesch, M.~Burri, S.~Omari, M.~Hutter, and R.~Siegwart, ``Iterated extended
  kalman filter based visual-inertial odometry using direct photometric
  feedback,'' \emph{The International Journal of Robotics Research}, vol.~36,
  no.~10, pp. 1053--1072, 2017.

\bibitem{roumeliotis2002icra}
S.~I. Roumeliotis and J.~W. Burdick, ``Stochastic cloning: A generalized
  framework for processing relative state measurements,'' in \emph{IEEE
  International Conference on Robotics and Automation}, Washington, D.C., May
  2002, pp. 1788--1795.

\bibitem{lupton2012visual}
T.~Lupton and S.~Sukkarieh, ``Visual-inertial-aided navigation for high-dynamic
  motion in built environments without initial conditions,'' \emph{IEEE
  transactions on robotics}, vol.~28, no.~1, pp. 61--76, 2012.

\bibitem{forster2015manifold}
C.~Forster, L.~Carlone, F.~Dellaert, and D.~Scaramuzza, ``Imu preintegration on
  manifold for efficient visual-inertial maximum-a-posteriori estimation,'' in
  \emph{Robotics: Science and Systems}, Rome, Italy, July 2015.

\bibitem{mourikis2009vision}
A.~I. Mourikis, N.~Trawny, S.~I. Roumeliotis, A.~E. Johnson, A.~Ansar, and
  L.~Matthies, ``Vision-aided inertial navigation for spacecraft entry,
  descent, and landing,'' \emph{IEEE transactions on robotics}, vol.~25, no.~2,
  pp. 264--280, 2009.

\bibitem{li2013optimization}
M.~Li and A.~I. Mourikis, ``Optimization-based estimator design for
  vision-aided inertial navigation,'' in \emph{Robotics: Science and Systems},
  Berlin, Germany, June 2013, pp. 241--248.

\bibitem{weiss2011real}
S.~Weiss and R.~Siegwart, ``Real-time metric state estimation for modular
  vision-inertial systems,'' in \emph{IEEE International Conference on Robotics
  and Automation}, Shanghai, China, May 2011, pp. 4531--4537.

\bibitem{kneip2011robust}
L.~Kneip, M.~Chli, and R.~Y. Siegwart, ``Robust real-time visual odometry with
  a single camera and an imu,'' in \emph{British Machine Vision Conference},
  September 2011.

\bibitem{indelman2013information}
V.~Indelman, S.~Williams, M.~Kaess, and F.~Dellaert, ``Information fusion in
  navigation systems via factor graph based incremental smoothing,''
  \emph{Robotics and Autonomous Systems}, vol.~61, no.~8, pp. 721--738, 2013.

\bibitem{triggs1999bundle}
B.~Triggs, P.~F. McLauchlan, R.~I. Hartley, and A.~W. Fitzgibbon, ``Bundle
  adjustment: a modern synthesis,'' in \emph{International Workshop on Vision
  Algorithms}, Corfu, Greece, September 1999, pp. 298--372.

\bibitem{breckenridge1979jplq}
W.~G. Breckenridge, ``Quaternions proposed standard conventions,'' NASA Jet
  Propulsion Laboratory, Tech. Rep., 1979.

\bibitem{trawny2005indirect}
N.~Trawny and S.~I. Roumeliotis, ``Indirect kalman filter for 3d attitude
  estimation,'' Department of Computer Science and Engineering, University of
  Minnesota, Tech. Rep., 2005.

\bibitem{Eckenhoff2016WAFR}
K.~Eckenhoff, P.~Geneva, and G.~Huang, ``High-accuracy preintegration for
  visual-inertial navigation,'' in \emph{International Workshop on the
  Algorithmic Foundations of Robotics}, San Francisco, CA, December 2016.

\bibitem{supp}
Z.~Huai and G.~Huang, ``Robocentric visual-inertial odometry,'' RPNG,
  University of Delaware, Tech. Rep., 2018,
  \url{http://udel.edu/~ghuang/papers/tr_rvio_ijrr.pdf}.

\bibitem{civera2008inverse}
J.~Civera, A.~J. Davison, and J.~M. Montiel, ``Inverse depth parametrization
  for monocular slam,'' \emph{IEEE transactions on robotics}, vol.~24, no.~5,
  pp. 932--945, 2008.

\bibitem{golub2012matrix}
G.~H. Golub and C.~F. Van~Loan, \emph{Matrix Computations}.\hskip 1em plus
  0.5em minus 0.4em\relax JHU Press, 2012, vol.~3.

\bibitem{Maybeck1979}
P.~S. Maybeck, \emph{Stochastic Models, Estimation, and Control}.\hskip 1em
  plus 0.5em minus 0.4em\relax London: Academic Press, 1979, vol.~1.

\bibitem{guo2013icra}
C.~Guo and S.~Roumeliotis, ``{IMU-RGBD} camera {3D} pose estimation and
  extrinsic calibration: Observability analysis and consistency improvement,''
  in \emph{IEEE International Conference on Robotics and Automation},
  Karlsruhe, Germany, May 2013, pp. 2935--2942.

\bibitem{hesch2014ijrr}
J.~Hesch, D.~Kottas, S.~Bowman, and S.~Roumeliotis, ``Camera-{IMU}-based
  localization: Observability analysis and consistency improvement,'' \emph{The
  International Journal of Robotics Research}, vol.~33, no.~1, pp. 182--201,
  2014.

\bibitem{chen1990local}
Z.~Chen, K.~Jiang, and J.~C. Hung, ``Local observability matrix and its
  application to observability analyses,'' in \emph{The 16th Annual Conference
  of IEEE Industrial Electronic Society}, Pacific Grove, CA, 1990, pp.
  100--103.

\bibitem{wu2017vins}
K.~J. Wu, C.~X. Guo, G.~Georgiou, and S.~I. Roumeliotis, ``Vins on wheels,'' in
  \emph{IEEE International Conference on Robotics and Automation}, Singapore,
  Singapore, May 2017, pp. 5155--5162.

\bibitem{Bar-Shalom2001}
Y.~Bar-Shalom, X.~R. Li, and T.~Kirubarajan, \emph{Estimation with applications
  to tracking and navigation}.\hskip 1em plus 0.5em minus 0.4em\relax New York:
  John Wiley \& Sons, 2001.

\bibitem{shi1994good}
J.~Shi and C.~Tomasi, ``Good features to track,'' in \emph{IEEE Computer
  Society Conference on Computer Vision and Pattern Recognition}, Seattle, WA,
  June 1994, pp. 593--600.

\bibitem{baker2004lucas}
S.~Baker and I.~Matthews, ``Lucas-kanade 20 years on: A unifying framework,''
  \emph{International Journal of Computer Vision}, vol.~56, no.~3, pp.
  221--255, 2004.

\bibitem{troiani20142}
C.~Troiani, A.~Martinelli, C.~Laugier, and D.~Scaramuzza, ``2-point-based
  outlier rejection for camera-imu systems with applications to micro aerial
  vehicles,'' in \emph{IEEE International Conference on Robotics and
  Automation}, Hong Kong, China, May 2014, pp. 5530--5536.

\bibitem{burri2016euroc}
M.~Burri, J.~Nikolic, P.~Gohl, T.~Schneider, J.~Rehder, S.~Omari, M.~W.
  Achtelik, and R.~Siegwart, ``The euroc micro aerial vehicle datasets,''
  \emph{The International Journal of Robotics Research}, vol.~35, no.~10, pp.
  1157--1163, 2016.

\end{thebibliography}

\section*{Appendix A: Bundle adjustment using inverse-depth parameterized landmark}

Assuming a single landmark, $L$, which has been observed from a set of consecutive robocentric frames in the sliding window, the set of corresponding camera frames is denoted by $\mathcal{C}$. To compute an inverse-depth estimate of $L$, i.e., $\boldsymbol{\lambda}=[\phi,\psi,\rho]^\top$, we use the proposed inverse-depth measurement model (see \eqref{eq:zinv}) ($i\in\mathcal{C}$):
\begin{align}
    \mathbf{z}_i &= \frac{1}{h_{i,3}(\phi,\psi,\rho)}
    \begin{bmatrix}
    h_{i,1}(\phi,\psi,\rho) \\ h_{i,2}(\phi,\psi,\rho)
    \end{bmatrix}+\mathbf{n}_i \nonumber \\
    &= \bar{\mathbf{z}}_i(\phi,\psi,\rho)+\mathbf{n}_i
\end{align}
where $\mathbf{n}_i\sim\mathcal{N}(\mathbf{0},\boldsymbol{\Lambda}_i)$ is the image noise, while the relative poses, $\mathbf{w}$, are assumed known.
Given the measurements $\mathbf{z}_i=(\frac{u_i}{f},\frac{v_i}{f})$, $i\in\{\mathcal{C}\backslash{1}\}$, we can formulate a bundle adjustment problem for solving $\boldsymbol{\lambda}$, as:
\begin{align}
    \boldsymbol{\lambda}^\ast &= \arg\min_{\boldsymbol{\lambda}}\sum_{i\in\{\mathcal{C}\backslash{1}\}}\big\lVert\bar{\mathbf{z}}_i(\boldsymbol{\lambda})-\mathbf{z}_i\big\rVert_{\boldsymbol{\Lambda}_i} \nonumber \\
    &=
    \arg\min_{\boldsymbol{\lambda}}\sum_{i\in\{\mathcal{C}\backslash{1}\}}\big\lVert\boldsymbol{\epsilon}_i(\boldsymbol{\lambda})\big\rVert_{\boldsymbol{\Lambda}_i}
\end{align}
where $\lVert\cdot\rVert_{\boldsymbol{\Lambda}}$ denotes the $\boldsymbol{\Lambda}$-weighted energy norm, and we define $\boldsymbol{\epsilon}_i$ as the residual associated to $\mathbf{z}_i$. This problem can be solved iteratively via Gauss-Newton approximation about the initial estimate of $\hat{\boldsymbol{\lambda}}$, as:
\begin{align}
    \delta{\boldsymbol{\lambda}}^\ast &=
    \arg\min_{\delta{\boldsymbol{\lambda}}}\sum_{i\in\{\mathcal{C}\backslash{1}\}}\big\lVert\boldsymbol{\epsilon}_i(\hat{\boldsymbol{\lambda}}+\delta{\boldsymbol{\lambda}})\big\rVert_{\boldsymbol{\Lambda}_i} \nonumber \\
    &\simeq
    \arg\min_{\delta{\boldsymbol{\lambda}}}\sum_{i\in\{\mathcal{C}\backslash{1}\}}\big\lVert\boldsymbol{\epsilon}_i(\hat{\boldsymbol{\lambda}})+\mathbf{H}_i\delta{\boldsymbol{\lambda}}\big\rVert_{\boldsymbol{\Lambda}_i}
    \label{eq:dlambda}
\end{align}
For the initial value of $\hat{\boldsymbol{\lambda}}$, we obtain $[\hat{\phi},\hat{\psi}]^\top$ by directly using the measurement of $L$, with the following equation:
\begin{equation}
    \begin{bmatrix}
    \hat{\phi} \\ \hat{\psi}
    \end{bmatrix} =
    \begin{bmatrix}
    \arctan\left(\frac{v_1}{f},\sqrt{(\frac{u_1}{f})^2+1}\right) \\
    \arctan\left(\frac{u_1}{f},1\right)
    \end{bmatrix}
\end{equation}
however, the initial value for $\hat{\rho}$ can be empirically chosen, 
for which we choose $0$ to put landmark at infinity first, and let it converge by performing iteration. The Jacobian of residual, $\mathbf{H}_i=\frac{\partial{\boldsymbol{\epsilon}_i(\hat{\boldsymbol{\lambda}}+\delta{\boldsymbol{\lambda}})}}{\partial{\delta{\boldsymbol{\lambda}}}}$, evaluated at $\hat{\boldsymbol{\lambda}}$ can be obtained following the chain rule, as:
\begin{equation}
    \mathbf{H}_i = \frac{\partial{\boldsymbol{\epsilon}_i}}{\partial{\mathbf{h}_i}}\frac{\partial{\mathbf{h}_i}}{\partial{\boldsymbol{\lambda}}}\frac{\partial{\boldsymbol{\lambda}}}{\partial{\delta{\boldsymbol{\lambda}}}} = \frac{\partial{\boldsymbol{\epsilon}_i}}{\partial{\mathbf{h}_i}}\frac{\partial{\mathbf{h}_i}}{\partial{\boldsymbol{\lambda}}} \nonumber
\end{equation}
where
\begin{align}
    \frac{\partial{\boldsymbol{\epsilon}_i}}{\partial{\mathbf{h}_i}} &= \frac{1}{\hat{h}_{i,3}}
    \begin{bmatrix}
    1 & 0 & -\frac{\hat{h}_{i,1}}{\hat{h}_{i,3}} \\
    0 & 1 & -\frac{\hat{h}_{i,2}}{\hat{h}_{i,3}}
    \end{bmatrix}, \nonumber \\
    \frac{\partial{\mathbf{h}_i}}{\partial{\boldsymbol{\lambda}}} &=
    \begin{bmatrix}
    \frac{\partial{\mathbf{h}_i}}{\partial{[\phi,\psi]^\top}} & \frac{\partial{\mathbf{h}_i}}{\partial{\rho}}
    \end{bmatrix} \nonumber \\
    &=
    \begin{bmatrix}
    {^i_1}\bar{\mathbf{C}}_{\hat{\bar{q}}}
    \begin{bmatrix}
    -\sin\hat{\phi}\sin\hat{\psi} & \cos\hat{\phi}\cos\hat{\psi} \\
    \cos\hat{\phi} & 0 \\
    -\sin\hat{\phi}\cos\hat{\psi} & -\cos\hat{\phi}\sin\hat{\psi}
    \end{bmatrix} &
    {^i}\hat{\bar{\mathbf{p}}}_1
    \end{bmatrix}
\end{align}
Every iteration we have the optimal inverse-depth correction, $\delta{\boldsymbol{\lambda}}^\ast$, and the estimate, $\hat{\boldsymbol{\lambda}}$, in the form of:
\begin{equation}
    \delta\boldsymbol{\lambda}^\ast = \left(\sum_{i\in\{\mathcal{C}\backslash{1}\}}\mathbf{H}_i^\top\boldsymbol{\Lambda}_i^{-1}\mathbf{H}_i\right)^{-1}\left(\sum_{i\in\{\mathcal{C}\backslash{1}\}}\mathbf{H}_i^\top\boldsymbol{\Lambda}_i^{-1}\boldsymbol{\epsilon}_i\right), \nonumber
\end{equation}
\begin{equation}
    \hat{\boldsymbol{\lambda}}\leftarrow\hat{\boldsymbol{\lambda}}+\delta{\boldsymbol{\lambda}}^\ast
\end{equation}
Once $\delta{\boldsymbol{\lambda}}^\ast$ gets converged (e.g., less than a threshold), we find the optimal inverse-depth estimate: $\boldsymbol{\lambda}^\ast = \hat{\boldsymbol{\lambda}}$.

\section*{Appendix B: Proof of Lemma 1}

Consider the case where the VINS estimation process is up to a scale factor, $s$ (that is, to recover the true state, $\mathbf{x}$, the underlying state, $\mathbf{x}'$, has to be``scaled up" metrically). This results in the following expressions of VINS states, in which the relative translation and landmark position with respect to $\{R_k\}$ can be written as (see \eqref{eq:x}):
\begin{align}
    {^{R_k}}\mathbf{p}_G &= s{^{R_k}}\mathbf{p}'_G \label{eq:spG} \\
    {^{R_k}}\mathbf{p}_I &= s{^{R_k}}\mathbf{p}'_I \label{eq:spI} \\
    {^{R_k}}\mathbf{p}_L &= s{^{R_k}}\mathbf{p}'_L \label{eq:spL}
\end{align}
where ${^{R_k}}\mathbf{p}'_G$, ${^{R_k}}\mathbf{p}'_I$ and ${^{R_k}}\mathbf{p}'_L$ are the values of underlying states. Note that the analysis presented in this proof holds true for any $t\in[t_k,t_{k+m}]$, hence we omit the time index for brevity of presentation. The scale change does not affect the rotation, as the scale $s$ corresponds to the translation only. Therefore, we have:
\begin{align}
    &\boldsymbol{\omega} = \boldsymbol{\omega}' \Rightarrow \nonumber \\
    &{^k_G}\mathbf{C}_{\bar{q}} = {^k_G}\mathbf{C}'_{\bar{q}}, \quad
    {^I_k}\mathbf{C}_{\bar{q}} = {^I_k}\mathbf{C}'_{\bar{q}}
    \label{eq:sqGI}
\end{align}
With those equations, the IMU velocity and acceleration can be obtained by taking the time derivative of \eqref{eq:spI}, as:
\begin{align}
    &{^I_k}\mathbf{C}_{\bar{q}}{^{R_k}}\mathbf{v}_I = s{^I_k}\mathbf{C}'_{\bar{q}}{^{R_k}}\mathbf{v}'_I \Rightarrow \nonumber \\
    &\mathbf{v}_I = s\mathbf{v}'_I, \quad
    {^I}\mathbf{a} = s{^I}\mathbf{a}'
    \label{eq:sva}
\end{align}
In particular, ${^{R_k}}\mathbf{g}$ is a state having known magnitude, thus is not affected by the scaling, i.e.,
\begin{equation}
    {^{R_k}}\mathbf{g} = {^{R_k}}\mathbf{g}'
    \label{eq:sg}
\end{equation}
Accordingly, the gravity effect to the IMU frame is estimated based on the local gravity, as:
\begin{align}
    &{^I_k}\mathbf{C}_{\bar{q}}{^{R_k}}\mathbf{g} = {^I_k}\mathbf{C}'_{\bar{q}}{^{R_k}}\mathbf{g} \Rightarrow \nonumber \\
    &{^I}\mathbf{g} = {^I}\mathbf{g}'
\end{align}
If such scale change is unobservable, then the measurements from the camera and IMU should remain the same. First, for the camera measurement of $L$ (see \eqref{eq:zl}), we have:
\begin{align}
    \begin{bmatrix}
    x \\ y \\ z
    \end{bmatrix} &=
    {^I}\mathbf{p}_L = {^I_k}\mathbf{C}_{\bar{q}}\big({^{R_k}}\mathbf{p}_L-{^{R_k}}\mathbf{p}_I\big) \nonumber \\
    &= s{^I_k}\mathbf{C}'_{\bar{q}}\big({^{R_k}}\mathbf{p}'_L-{^{R_k}}\mathbf{p}'_I\big) = s{^I}\mathbf{p}'_L = s
    \begin{bmatrix}
    x' \\ y' \\ z'
    \end{bmatrix} \Rightarrow \nonumber \\
    &\mathbf{z}_I = 
    \frac{1}{z}
    \begin{bmatrix}
    x \\ y
    \end{bmatrix} =
    \frac{1}{sz'}
    \begin{bmatrix}
    sx' \\ sy'
    \end{bmatrix} = 
    \frac{1}{z'}
    \begin{bmatrix}
    x' \\ y'
    \end{bmatrix} = 
    \mathbf{z}'_I
\end{align}
where the camera measurement does not change because the scale is invariant for perspective projection model. Then, for the IMU measurements we first examine the angular velocity measured by the gyroscope (see \eqref{eq:wm}), as:
\begin{align}
    &\boldsymbol{\omega}_m = \boldsymbol{\omega}+\mathbf{b}_g = \boldsymbol{\omega}'+\mathbf{b}'_g \Rightarrow \nonumber \\
    &\mathbf{b}_g = \mathbf{b}'_g
    \label{eq:sbg}
\end{align}
Similarly, for the linear acceleration measurements from the accelerometer (see \eqref{eq:am}), we have:
\begin{align}
    &\mathbf{a}_m = {^I}\mathbf{a}+{^I}\mathbf{g}+\mathbf{b}_a = {^I}\mathbf{a}'+{^I}\mathbf{g}'+\mathbf{b}'_a \Rightarrow \nonumber \\
    &\mathbf{b}_a = \mathbf{b}'_a-(s-1){^I}\mathbf{a}'
    \label{eq:sba}
\end{align}
Note that, $\mathbf{b}_a$ cannot be simply represented as the multiple of $\mathbf{b}'_a$, because it is a random walk process (see \eqref{eq:xdot}). Thus, based on \eqref{eq:spG}, \eqref{eq:spI}, \eqref{eq:spL}, \eqref{eq:sqGI}, \eqref{eq:sva}, \eqref{eq:sg}, \eqref{eq:sbg}, and \eqref{eq:sba}, it is not difficult to validate the corresponding error-state relation as shown in \eqref{eq:sx}.

\section*{Appendix C: Proof of Lemma 2}

Based on the observability matrix (see \eqref{eq:obsM}), the $\ell$-th block row, $\mathbf{M}'_\ell$, of observability matrix $\mathbf{M}'$ evaluating at ${^{R_k}}\hat{\mathbf{x}}'_\ell$ and ${^{R_k}}\hat{\mathbf{p}}'_L$, has the following structure:
\begin{equation}
    \mathbf{M}'_\ell = 
    \boldsymbol{\Pi}'
    \begin{bmatrix}
    \mathbf{0}_{3\times6} & \boldsymbol{\Gamma}'_1 & \boldsymbol{\Gamma}'_2 & -\mathbf{I}_3 & \boldsymbol{\Gamma}'_3 & \boldsymbol{\Gamma}'_4 & \boldsymbol{\Gamma}'_5 \;\; \big\lvert \;\; \mathbf{I}_3
    \end{bmatrix} \nonumber
\end{equation}
The direction of scale, $\mathbf{u}$, is unobservable (see \eqref{eq:sx}), if and only if $\mathbf{M}'_\ell\mathbf{u}=\mathbf{0}$, $\forall \ell \geq k$, thus we have:
\begin{equation}
    \boldsymbol{\Pi}'\big(
    -{^{R_k}}\hat{\mathbf{p}}'_{I_\ell}+\boldsymbol{\Gamma}'_3\hat{\mathbf{v}}'_{I_\ell}-\boldsymbol{\Gamma}'_5{^\ell}\hat{\mathbf{a}}'+{^{R_k}}\hat{\mathbf{p}}'_L
    \big)
    = \mathbf{0}
    \label{eq:M0}
\end{equation}
where
\begin{equation}
    \boldsymbol{\Pi}'\big({^{R_k}}\hat{\mathbf{p}}'_L-{^{R_k}}\hat{\mathbf{p}}'_{I_\ell}\big) = \mathbf{H}'_\text{p}{^{I_\ell}}\hat{\mathbf{p}}'_L \nonumber = \mathbf{0}
\end{equation}
because ${^{I_\ell}}\hat{\mathbf{p}}'_L$ is in the right nullspace of $\mathbf{H}'_\text{p}$ (see \eqref{eq:zl} and \eqref{eq:Jobs}). Then, what is left to show is:
\begin{equation}
    \boldsymbol{\Pi}'\big(\boldsymbol{\Gamma}'_3\hat{\mathbf{v}}'_{I_\ell}-\boldsymbol{\Gamma}'_5{^\ell}\hat{\mathbf{a}}'\big) = \mathbf{0}
    \label{eq:cond0}
\end{equation}
where
\begin{equation}
    \boldsymbol{\Gamma}'_3\hat{\mathbf{v}}'_{I_\ell}-\boldsymbol{\Gamma}'_5{^\ell}\hat{\mathbf{a}}'
    = -\Delta{t}_{k,\ell}\hat{\mathbf{v}}'_{I_\ell}-\int_{t_k}^{t_\ell}\int_{t_k}^{\tau}{^\mu_k}\mathbf{C}_{\hat{\bar{q}}}^\top\;{d\mu}{d\tau}{^\ell}\hat{\mathbf{a}}' \nonumber
\end{equation}
To this end, we examine two special cases: (i) if no rotations (i.e., $\boldsymbol{\omega}=\mathbf{0}$, $\forall \tau\in[t_k,t_\ell]$), then we have:
\begin{align}
    \boldsymbol{\Gamma}'_3\hat{\mathbf{v}}'_{I_\ell}-\boldsymbol{\Gamma}'_5{^\ell}\hat{\mathbf{a}}'
    &= -\Delta{t}_{k,\ell}\hat{\mathbf{v}}'_{I_\ell}-\int_{t_k}^{t_\ell}\int_{t_k}^{\tau}\mathbf{I}_3\;{d\mu}{d\tau}{^\ell}\hat{\mathbf{a}}' \nonumber \\
    &= -\Delta{t}_{k,\ell}\hat{\mathbf{v}}'_{I_\ell}-\frac{1}{2}\Delta{t}_{k,\ell}^2{^\ell}\hat{\mathbf{a}}'
    \label{eq:cond1}
\end{align}
and (ii) if constant local acceleration (i.e., ${^\tau}\mathbf{a}'\equiv{^k}\mathbf{a}'$, $\forall \tau\in[t_k,t_\ell]$), then we have:
\begin{align}
    \boldsymbol{\Gamma}'_3\hat{\mathbf{v}}'_{I_\ell}-\boldsymbol{\Gamma}'_5{^\ell}\hat{\mathbf{a}}' &= -\Delta{t}_{k,\ell}\hat{\mathbf{v}}'_{I_\ell}-\int_{t_k}^{t_\ell}\int_{t_k}^{\tau}{^\mu_k}\mathbf{C}_{\hat{\bar{q}}}^\top{^\mu}\hat{\mathbf{a}}'\;{d\mu}{d\tau} \nonumber \\
    &= -\Delta{t}_{k,\ell}\hat{\mathbf{v}}'_{I_\ell}-\int_{t_k}^{t_\ell}\int_{t_k}^{\tau}{^{R_k}}\hat{\mathbf{a}}'(\mu)\;{d\mu}{d\tau} \nonumber \\
    &= -\Delta{t}_{k,\ell}\hat{\mathbf{v}}'_{I_\ell}-\int_{t_k}^{t_\ell}\big({^{R_k}}\hat{\mathbf{v}}'_{I_\tau}-{^{R_k}}\hat{\mathbf{v}}'_{I_k}\big){d\tau} \nonumber \\
    &= -\Delta{t}_{k,\ell}\hat{\mathbf{v}}'_{I_\ell}-{^{R_k}}\hat{\mathbf{p}}'_{I_\ell}
    \label{eq:cond2}
\end{align}
To ensure that \eqref{eq:cond0} holds, both \eqref{eq:cond1} and \eqref{eq:cond2} should be equal to $\mathbf{0}$, and the conclusion of Lemma 2 is immediate.

\end{document}